\documentclass[a4paper,12pt]{article}

\usepackage{amsmath}
\usepackage{amssymb}
\usepackage{amsthm}
\usepackage[english]{babel}
\usepackage{caption}
\usepackage{color}
\usepackage{enumerate}
\usepackage{graphicx}
\usepackage{hyperref}
\usepackage{subcaption}
\usepackage{verbatim}
\usepackage{xspace}

\usepackage[authoryear]{natbib} 
\renewcommand{\cite}[1]{\citep{#1}} 
\newcommand{\simaWcitet}{\citet{sima_neuromata1998}}

\newcommand{\kappelcitet}{\citet{kappel_markov2014}}

\theoremstyle{definition}
\newtheorem{theorem}{Theorem}[section]
\newtheorem{example}[theorem]{Example}
\newtheorem{lemma}[theorem]{Lemma}

\newtheorem{remark}[theorem]{Remark}
\newtheorem{claim}[theorem]{Claim}

\numberwithin{equation}{section}

\newcommand{\fname}[1]{\mathit{#1}} 
\newcommand{\romI}{\emph{(i)}} 
\newcommand{\romII}{\emph{(ii)}} 
\newcommand{\romIII}{\emph{(iii)}} 
\renewcommand{\qed}{\hfill$\square$} 

\newcommand{\sima}{\v{S}\'{\i}ma}
\newcommand{\simaW}{\sima{} and Wiedermann}

\newcommand{\kappel}{Kappel~et~al.}

\newcommand{\interval}[2]{\left[#1,#2\right]}
\newcommand{\nat}{\mathbb{N}}
\newcommand{\natzero}{\mathbb{N}_0} 
\newcommand{\powset}[1]{\mathcal{P}(#1)}
\newcommand{\real}{\mathbb{R}}
\newcommand{\set}[1]{\{#1\}}
\newcommand{\ssize}[1]{\left|#1\right|}
\newcommand{\mult}{\mathbin{\cdot}}
\newcommand{\myfrac}[2]{{#1}/{#2}} 
\newcommand{\braces}[1]{\left(#1\right)}
\newcommand{\jset}{X} 

\newcommand{\jump}{\epsilon}
\newcommand{\lang}{\mathcal{L}}
\newcommand{\len}[1]{\left|#1\right|} 
\newcommand{\str}{\alpha}
\newcommand{\strB}{\beta}
\newcommand{\strC}{\gamma}
\newcommand{\mstr}[1]{(#1)} 
\newcommand{\sym}[1]{S_{#1}} 
\newcommand{\symm}{S} 
\newcommand{\prefix}[2]{#1_{\rightarrow#2}}

\newcommand{\nfaX}[1]{M_{#1}}
\newcommand{\statesX}[1]{Q_{#1}}
\newcommand{\alpX}[1]{\Sigma_{#1}}
\newcommand{\startstateX}[1]{q^\mathrm{s}_{#1}}
\newcommand{\acceptstatesX}[1]{F_{#1}}
\newcommand{\trfX}[1]{\delta_{#1}} 
\newcommand{\nfatupX}[1]{(%
    \statesX{#1},%
    \alpX{#1},%
    \trfX{#1},%
    \startstateX{#1},%
    \acceptstatesX{#1})}

\newcommand{\nfa}{\nfaX{}}
\newcommand{\states}{\statesX{}}
\newcommand{\alp}{\alpX{}}
\newcommand{\startstate}{\startstateX{}}
\newcommand{\acceptstates}{\acceptstatesX{}}
\newcommand{\trf}{\trfX{}}
\newcommand{\nfatup}{\nfatupX{}}

\newcommand{\pset}{P} 

\newcommand{\be}{\boldsymbol{B}}
\newcommand{\nrl}[1]{\lang(#1)} 
\newcommand{\monoregexpr}[1]{#1\in\be(\str)\Leftrightarrow\text{$\str$ embeds a string $\strB\in\nrl{#1}$}}

\newcommand{\nw}{\boldsymbol{\mathcal{N}}}
\newcommand{\inr}{\mathcal{I}}
\newcommand{\onr}{\mathcal{O}}
\newcommand{\aux}{\mathcal{A}}
\newcommand{\weightinit}{\mathcal{W}}
\newcommand{\nwtup}{(%
    \inr,%
    \onr,%
    \aux,%
    \weightinit%
    )}
\newcommand{\conn}[1]{\fname{edges}(#1)}
\newcommand{\post}[1]{\fname{post}(#1)}
\newcommand{\pre}[1]{\fname{pre}(#1)}

\newcommand{\activeX}[1]{N_{#1}}
\newcommand{\run}{\mathcal{R}}

\newcommand{\haltstate}{q^\mathrm{h}}

\newcommand{\loops}{V}

\newcommand{\suitable}{clean}
\newcommand{\Suitable}{Clean}
\newcommand{\useful}{reachable}
\newcommand{\pairset}[1]{p(#1)}
\newcommand{\paircount}[1]{c(#1)}
\newcommand{\nfaref}{\boldsymbol{\mathcal{M}}}
\newcommand{\implementation}{automaton implementation}

\newcommand{\worname}{w_1} 
\newcommand{\wandname}{w_2} 
\newcommand{\wor}[2]{\worname(#1,#2)}
\newcommand{\wand}[2]{\wandname(#1,#2)}
\newcommand{\triggers}{\mathcal{T}} 
\newcommand{\context}{\fname{con}} 
\newcommand{\qsX}[1]{(q_{#1},\sym{#1})}

\newcommand{\conv}{V}
\newcommand{\outcontext}{C}

\begin{document}

\title{
    Positive Neural Networks in Discrete Time Implement Monotone-Regular Behaviors
}
\author{%
    Tom~J.~Ameloot%
        \thanks{T.J.~Ameloot is a Postdoctoral Fellow of the Research Foundation -- Flanders (FWO).}~   
    and
    Jan~Van~den~Bussche
    \\
    {\small Hasselt University \& transnational University of Limburg}
}
\date{}

\maketitle{}

    \textbf{Abstract.}
    We study the expressive power of positive neural networks. 
    The model uses positive connection weights and multiple input neurons.     
    Different behaviors can be expressed by varying the connection weights.   
    We show that in discrete time, and in absence of noise, the class of positive neural networks captures the so-called monotone-regular behaviors, that are based on regular languages.
    A finer picture emerges if one takes into account the delay by which a monotone-regular behavior is implemented.     
    Each monotone-regular behavior can be implemented by a positive neural network with a delay of one time unit.
    Some monotone-regular behaviors can be implemented with zero delay. 
    And, interestingly, some simple monotone-regular behaviors can not be implemented with zero delay.  



\section{Introduction}

\paragraph*{Positive neural networks}

Based on experimental observations, \citet{douglas_circuits2004} have proposed an abstract model of the neocortex consisting of interconnected winner-take-all circuits.
Each winner-take-all circuit consists of excitatory neurons that, besides exciting each other, indirectly inhibit each other through some inhibition layer. This causes only a few neurons in the circuit to be active at any time.
\citet{kappel_markov2014} further demonstrate through theoretical analysis and simulations that the model of interconnected winner-take-all circuits might indeed provide a deeper understanding of some experimental observations.
In this article, we take two inspirational points from this model, that we discuss below.

First, although a biological neural network in general contains both excitatory and inhibitory connections between neurons~\cite{gerstner_book2014}, excitation and inhibition are not combined in an arbitrary fashion in the above model of interconnected winner-take-all circuits. In that model, the meaning seems to be mostly contained in the excitatory connections whereas inhibitory connections play a more regulatory role such as controlling how many neurons can become simultaneously active.
Based on this apparent value of excitatory connections, in this article we are inspired to study neural networks that are simplified to contain only excitatory connections between neurons.
Technically, we consider so-called positive neural networks, where all connections are given a weight that is either strictly positive or zero.

Second, it appears useful to study neural network models with multiple input neurons. 
In the model of interconnected winner-take-all circuits, each circuit has multiple input neurons that can be concurrently active. This allows each circuit to receive rich input symbols. The input neurons of a circuit could receive stimuli directly from sensory organs, or from other circuits.
It would be fascinating to understand how neurons build concepts or recognize patterns over such rich inputs.

Based on the above inspiration, in this article we study a simple positive neural network model with multiple input neurons, operating in discrete time.
As mentioned above, the use of nonnegative weights allows only excitation between neurons and no inhibition.
In our model, each positive neural network has distinguished sets of input neurons, output neurons, and auxiliary neurons.
The network may be recurrent, i.e., the activation of a neuron may indirectly influence its own future activation.
As in some previous models~\cite{sima_neuromata1998,sima_survey2003}, we omit noise and learning.
We believe that the omission of inhibition (i.e., negative connection weights) might allow for a better understanding of the foundations of computation in neural networks, where different features gradually increase the expressive power (see also Section~\ref{sec:conclusion}). 
Excitation between neurons seems to be a basic feature that we can not omit.
The omission of inhibition leads to a notion of monotonicity that we will discuss later in the Introduction.

As a final point for the motivation of the model, we mention that biological neurons seem to mostly encode information in the timing of their activations and not in the magnitude of the activation signals~\cite{gerstner_book2014}. 
In this perspective, one may view discrete time models like ours as highlighting the causal steps of the neuronal computation.
The discrete time step could in principle be chosen very small.

\paragraph*{Expressivity study}

Our aim in this article is to better understand what computations can or cannot be performed by positive neural networks. 
We show that positive neural networks represent the class of so-called \emph{monotone-regular behaviors}.  The relevance of this result is discussed later in the Introduction. We first provide the necessary context.

Many previous works have investigated the expressive power of various kinds of neural network models~\cite{sima_survey2003}.
A common idea is to relate neural networks to theoretical computation devices like automata~\cite{hopcroft-ullman1979,sipser_book2006}.
A lower bound on the expressiveness of a neural network model can be established by simulating automata with neural networks in that model.
Conversely, an upper bound on the expressiveness can be established by simulating neural networks with automata.
In previous works, simulations with neural networks of both deterministic finite automata~\cite{alon_1991,indyk_1995,omlin_1996,horne_fsm1996} and nondeterministic finite automata~\cite{carrasco_nfa1999} have been studied.
Some models of neural networks even allow the simulation of arbitrary Turing machines, that are much more powerful than finite automata, see e.g.~\cite{siegelmann_turing1995,maass_lowerbounds1996}. However, the technical constructions used for the simulation of such powerful machines are not necessarily biologically plausible.

In this article, our approach in the expressivity study is to describe the behaviors exhibited by positive neural networks, as follows.
An input symbol in our model is a subset of input neurons that are concurrently active. For example, if the symbol $\set{a,b,c}$ is presented as input to the network at time $t$ then this means that $a$, $b$, and $c$ are the only input neurons that are (concurrently) active at time $t$. 
The empty symbol would mean that no input neurons are active. 
Output symbols are defined similarly, but over output neurons instead.
Now, we define behaviors as functions that transform each sequence of input symbols to a sequence of output symbols.
Different behaviors can be expressed by varying the connection weights of a positive neural network.
By describing such behaviors, we can derive theoretical upper and lower bounds on the expressivity of positive neural networks.
We emphasize that we feed sequences of input symbols to the neural networks, not single symbols.

Our assumption of multiple input neurons that may become concurrently active is in contrast to models of past expressivity studies, where either input encodings were used that \romI\ made only one input neuron active at any given time~\cite{sima_neuromata1998,carrasco_nfa1999,carrasco_stable2000} or \romII\ presented a single bit string just once over multiple input neurons after which they remained silent~\cite{sima_survey2003}.
Essentially, multiple parallel inputs versus a single sequential input is a matter of input alphabet. One might propose that an external process could transform multiple parallel inputs to a single sequential one (say, a stream of bits), after which previous results might be applied, e.g.~\cite{sima_neuromata1998}. However, in a biologically plausible setting there is no such external process: in general, it seems that inputs arrive from multiple sensory organs in parallel, and an internal circuit receives inputs from multiple other internal circuits in parallel as remarked at the beginning of the Introduction.
Because our aim is to better understand the biologically plausible setting, we therefore have to work with an input alphabet where multiple (parallel) input neurons may become concurrently active.

\paragraph*{Monotone-regular behaviors}

To describe the behaviors exhibited by positive neural networks, we use the class of regular languages, which are those languages recognized by finite automata~\cite{hopcroft-ullman1979,sipser_book2006}.
Previously, \simaWcitet{} have shown that neural networks in discrete time that read bit strings over a single input neuron recognize whether prefixes of the input string belong to a regular language or not.
In their technical construction, \simaW\ essentially simulate nondeterministic finite automata.

In this article, we simulate nondeterministic finite automata in the setting of positive neural networks.%
    \footnote{For example, the finite automaton in Figure~\ref{fig:nfa} is simulated by the positive neural network in Figure~\ref{fig:nw}.}
Using the simulation, we show that the class of positive neural networks captures the so-called monotone-regular behaviors.
A monotone-regular behavior describes the activations of each output neuron with a regular language over input symbols, where each symbol may contain multiple input neurons as described above.
Monotonicity means that each output neuron is activated whenever strings of the regular language are embedded in the input, regardless of any other activations of input neurons. 
Phrased differently, enriching an input with more activations of input neurons will never lead to fewer activations of output neurons.
Monotonicity arises because neurons only excite each other and do not inhibit each other.
This notion did not appear explicitly in the work by \simaWcitet{} because their neural networks exactly recognize regular languages over the single input neuron by using inhibition (i.e., negative connection weights): inhibition allows to explicitly test for the absence of input activations at certain times.

Delay is a standard notion in the study of neural networks~\cite{sima_survey2003}. 
Intuitively, delay is the number of extra time steps needed by the neural network before it can produce the output symbols prescribed by the behavior.
%
We show that each monotone-regular behavior can be implemented by a positive neural network with a delay of one time unit.
This result is in line with the result by \simaWcitet{}, but it is based on a new technical construction to deal with the more complex input symbols generated by concurrently active input neurons.
We simulate automaton states by neurons as expected, but we design the weights of the incoming connections to a neuron to express simultaneously \romI\ an ``or'' over context neurons that provide working memory and \romII\ an ``and'' over all input neurons mentioned in an input symbol.
As in the work by \simaWcitet{}, the constructed neural network may activate auxiliary neurons in parallel. Accordingly, our simulation preserves the nondeterminism, or parallelism, of the simulated automaton.
As an additional result, we show that a large class of monotone-regular behaviors can be implemented with zero delay.
And, interestingly, some simple monotone-regular behaviors can provably not be implemented with zero delay.

To the best of our knowledge, the notion of monotone-regular behaviors is introduced in this article for the first time. But this notion is a natural combination of some previously existing concepts, results, and general intuition, as follows. 
First, it is likely that both the temporal structure and spatial structure of sensory inputs are important for biological organisms~\cite{buonomano2009}. The temporal structure describes the timing of sensory events, and the spatial structure describes which and how many neurons are used to represent each sensory event. 
Second, the well-known class of regular languages from formal language theory describes symbol sequences that exhibit certain patterns or regularities~\cite{hopcroft-ullman1979}; temporal structure is represented by the ordering of symbols, and spatial structure is given by the individual symbols. 
The relationship between regular languages and neural network models has also been investigated before~\cite{sima_neuromata1998}. 
Third, without inhibition, neurons only excite each other and therefore an increased activity of input neurons will not lead to a decreased activity of output neurons. Without inhibition, neurons will respond to patterns embedded in the input stream regardless of any other simultaneous patterns, giving rise to a form of monotonicity on the resulting behavior.

\paragraph*{Relevance}
We conclude the Introduction by placing our result in a larger picture.
The intuition explored in this article, is that neural networks in some sense represent grammars.
A grammar is any set of rules describing how to form sequences of symbols over a given alphabet; such sequences may be called sentences.

In an experiment by \citet{reber_grammar1967}, subjects were shown sentences generated by an artificial grammar, but the rules of the grammar were not shown. Subjects were better at memorizing and reproducing sentences generated by the grammar when compared to sentences that are just randomly generated. Moreover, subjects were generally able to classify sentences as being grammatical or not. Interestingly, however, subjects could not verbalize the underlying rules of the grammar. This experiment suggests that organisms learn patterns from the environment when the patterns are sufficiently repeated. Those patterns get embedded into the neural network. The resulting grammar can not necessarily be described or explicitly accessed by the organism.
     
The grammar hypothesis is to some extent confirmed by neuronal recordings of brain areas involved with movement planning in monkeys~\cite{shima_motor2000,isoda_saccade2003}. These experimental findings suggest that movement sequences are represented by two groups of neurons: the first group represents the temporal structure, and the second group represents individual actions. Neurons in the first group might be viewed as stringing together the output symbols represented by the second group. Hence, the first group might represent the structure of a grammar, indicating the allowed sentences of output symbols.

The above experiments are complemented by \kappelcitet{}, who have theoretically shown and demonstrated with computer simulations that neural winner-take-all circuits can (learn to) express hidden Markov models.
Hidden Markov models are finite state machines with transition probabilities between states, and each state has a certain probability to emit symbols. Such models describe grammars, because each visited state can contribute symbols to an increasing sentence.
One of the insights by \kappel\ is that by repeatedly showing sentences generated by a hidden Markov model to a learning winner-take-all circuit, the states of the Markov model are eventually encoded by global network states, i.e., groups of activated neurons.
This way, the neural network builds an internal model of how sentences are formed by the hidden grammar.
Interestingly, the computer simulations by \kappelcitet{} clearly demonstrate (and visualize) that neurons learn to cooperate in a chain-like fashion, expressing the symbol chains in the hidden grammar. This corresponds well to the earlier predictions~\cite{reber_grammar1967,shima_motor2000,isoda_saccade2003}.
We might speculate that, if one assumes a real-world environment to be a (complex) hidden Markov model, organisms with a neural network can learn to understand the patterns, or sentences, generated by that environment.

In this article, we have made the above grammar intuition formal for positive neural networks. By characterizing the expressive power of positive neural networks with monotone-regular behaviors, the activation of an output neuron may be viewed as the recognition of a pattern in the input. This way, each output neuron represents a grammar: the output neuron recognizes which input sentences satisfy the grammar.
Moreover, our finding that nondeterministic finite automata can be simulated by positive neural networks is in line with the expressivity result of \kappelcitet{} because hidden Markov models generalize nondeterministic automata~\cite{dupont_hmm2005}: in a standard nondeterministic automaton, all successor states of a given state are equally likely, whereas a hidden Markov model could assign different transition probabilities to each successor state.
The simulation of automata by previous works~\cite{sima_survey2003} and the current article, combined with the result by \kappelcitet{}, might provide a useful intuition: individual neurons or groups of neurons could represent automaton states of a grammar.

\paragraph*{Outline}
This article is organized as follows.
We discuss related work in Section~\ref{sec:relwork}.
We provide in Section~\ref{sec:prelim} the necessary preliminaries, including the formalization of positive neural networks and monotone-regular behaviors.
Next, we provide in Section~\ref{sec:results} our results regarding the expressivity of positive neural networks.
We conclude in Section~\ref{sec:conclusion} with topics for future work.

\section{Related Work}
\label{sec:relwork}

We now discuss several theoretical works that are related to this article.

The relationship between the semantic notion of monotonicity and the syntactic notion of positive weights is natural, and has been explored in other settings than the current article, see e.g.~\cite{beimel_monotone2006,legenstein_sign2008,daniels_monotone2010}. 
In particular, the paper by \citet{legenstein_sign2008} studies more generally the classification ability of sign-constrained neurons. In their setting, fixing some natural number $n$, there is one output neuron that is given points from $\real^n$ as presynaptic input. Each choice of weights from $\real^n$ allows the output neuron to express a binary (true-false) classification of input points, where ``true'' is represented by the activation of the output neuron.
By imposing sign-constraints on the weights, different families of output neurons are created.
For example, one could demand that only positive weights are used.
It turns out that the VC-dimension of sign-constrained neurons with $n$ presynaptic inputs is $n$, which is only one less than unconstrained neurons.%
    \footnote{For example, for the case of positive weights, the VC-dimension $n$ tells us that there is an input set $S\subseteq\real^n$ with $\ssize S=n$ that can be \emph{shattered} by the family of positive presynaptic weights, in the following sense: for each classification $h:S\to\set{1,0}$, there exists a positive presynaptic weight vector in $\real^n$ allowing the resulting single output neuron to express $h$ on $S$.}
Moreover, \citet{legenstein_sign2008} characterize the input sets (containing $n$ points from $\real^n$) for which sign-constrained neurons can express all binary classification functions.

Like in the Introduction, we define an input symbol as a set of concurrently active input neurons.
The results by \citet{legenstein_sign2008} can be used to better understand the nature of input symbols that are distinguishable from each other by a single output neuron having nonnegative presynaptic weights, also referred to as a positive neuron. 
Indeed, if we would receive a stream of input symbols and if we would like to individually classify each input symbol by the activation behavior of a positive output neuron (where activation means ``true''), the results by \citet{legenstein_sign2008} provide sufficient and necessary conditions on the presented input symbols to allow the output neuron to implement the classification.
It is also possible, however, to consider a temporal context for each input symbol: then, the decision to activate an output neuron for a certain input symbol depends on which input symbols were shown previously.
For example, considering an alphabet of input symbols $A$, $B$, and $C$, we might want to activate the output neuron on symbol $B$ while witnessing the string $\mstr{A,A,B}$ but not while only witnessing the string $\mstr{C,C,B}$. Hence, the output activation for symbol $B$ depends on the temporal context in which $B$ appears.
For a maximum string length $k$, and assuming that input symbols have maximum size $n$, one could in principle present a string of $k$ symbols to the output neuron in a single glance, using $n$ times $k$ (new) input neurons. In that approach, the output neuron could even recognize strings of a length $l\leq k$ that are embedded into the presented string, giving rise to the notion of monotonicity discussed in this article. For such cases, the results by \citet{legenstein_sign2008} could still be applied to better understand the nature of strings that can be recognized by a positive output neuron.
As mentioned in the Introduction, in this article we use regular languages to describe the strings of input symbols upon which an output neuron should become activated. In contrast to fixing a maximum length $k$ on strings, regular languages can describe arbitrarily long strings, by allowing arbitrary repetitions of substrings. By applying our expressivity lower bound (Theorem~\ref{theo:lower}), we can for example construct a neural network that activates an output neuron whenever a string of the form $A^*B$ is embedded at the end of the so-far witnessed stream of input symbols, where $A^*$ denotes that symbol $A$ may be repeated an arbitrary number of times. Moreover, the neural network can be constructed in such a way that the output neuron responds with a delay of at most one time unit compared to the pattern's appearance.
Results regarding regular languages, in combination with delay, can be analyzed in the framework of the current article. Instead of single output neurons, we consider larger networks where auxiliary neurons can assist the output neurons by reasoning over the temporal context of the input symbols.

Positive neural networks are also related to the monotone acyclic AND-OR boolean circuits, studied e.g.\ by \citet{alon_monotone1987}. Concretely, an AND-OR circuit is a directed acyclic graph whose vertices are gates that compute either an OR or an AND of the boolean signals generated by the predecessor gates. The input to the circuit consists of a fixed number of boolean variables. Each AND-OR circuit is a special case of a positive neural network: each AND and OR gate can be translated to a neuron performing the same computation, by applying positive edge weights to presynaptic neurons.

The neurons studied in the current article compute a Boolean linear threshold function of their presynaptic inputs: each neuron computes a weighted sum of the (Boolean) activations of its presynaptic neurons and becomes activated when that sum passes a threshold. Now, the acyclic AND-OR-NOT circuits discussed by \citet{parberry_book} are related to the AND-OR circuits mentioned above.%
    \footnote{\citet{parberry_book} actually refers to AND-OR-NOT circuits as AND-OR circuits because NOT gates can be pushed to the first layer, which can be used to establish a normal form where layers of AND gates alternate with layers of OR gates (with negation only at the first level).}
It turns out that every Boolean linear threshold function over $n$ input variables can be computed by an acyclic AND-OR-NOT circuit with a number of gates that is polynomial in $n$ and with a depth that is logarithmic in $n$.%
    \footnote{In particular, we are referring to Theorem~7.4.7 (and subsequently Corollary~6.1.6) of \citet{parberry_book}.}
One may call such circuits ``small'', although not of constant size in $n$.
The essential idea in this transformation, is that AND-OR-NOT circuits can compute the sum of the weights for which the corresponding presynaptic input is true, and subsequently compare that sum to a threshold; a binary encoding of the weights and threshold can be embedded into the circuit, but care is taken to ensure that this encoding is of polynomial size in $n$.
It appears, however, that NOT gates play a crucial role in the construction, for handling the carry bit in the summation.
The resulting circuit is therefore not positive (or monotone) in the sense of \citet{alon_monotone1987}.
For completeness, we remark that delay is increased if one would replace each Boolean linear threshold neuron with a corresponding AND-OR-NOT sub-circuit, at least if one time unit is consumed for calculating each gate of each sub-circuit. Given that the transformation produces sub-circuits of non-constant depth, it appears nontrivial to describe the overall delay exhibited by the network.

\citet{horne_fsm1996} show upper and lower bounds on the number of required neurons for simulating deterministic finite automata that read and write sequences of bits. Their approach is to encode the state transition function of an automaton as a Boolean function, that is subsequently implemented by an acyclic neural network.%
    \footnote{If there are $m$ automaton states then each state can be represented by $\lceil\log_2m\rceil$ bits. 
    }
Each execution of the entire acyclic neural network corresponds to one update step of the simulated automaton.
A possible advantage of the method by \citet{horne_fsm1996}, is that the required number of neurons could be smaller than the number of automaton states. But, like in the discussion of AND-OR-NOT circuits above, the construction introduces a nontrivial delay in the simulation of the automaton if each neuron (or each layer of neurons) is viewed as consuming one time unit.
In this article we are not necessarily concerned with compacting automaton states in as few neurons as possible, but we are instead interested in recognizing a regular language under a maximum delay constraint (of one time unit) in the setting where multiple input neurons can be concurrently active and produce a stream of complex input symbols.
The construction by \citet{horne_fsm1996} can be modified to multiple input neurons that may become concurrently active.

For completeness, we remark that in this article we do not impose the restriction that the (simulated) automata are deterministic.
Moreover, in our simulation of automata, we take care to only introduce a polynomial increase in the number of neurons compared to the original number of automaton states (see~Theorem~\ref{theo:lower}).
In particular, if the original automaton is nondeterministic, the constructed neural network for this automaton will preserve that nondeterminism in the form of concurrently active neurons.
This stems from our original motivation to propose a construction that could in principle be biologically plausible, where multiple neurons could be active in parallel to explore the different states of the original automaton.
From this perspective, implementing a deterministic solution, where only one neuron is active at any given moment, would be less interesting.

Monotonicity in the context of automata has appeared earlier in the work by \citet{gecseg_monotone2001}. There, an automaton is called \emph{monotone} if there exists a partial order $\leq$ on the automaton states, such that each transition $(a,x,a')$, going from state $a$ to state $a'$ through symbol $x$, satisfies $a\leq a'$. 
Intuitively, this condition prohibits cycles between two different states while parsing a string. In particular, the same state $a$ may not be reused except when the previous state was already $a$ (i.e., self-looping on $a$ is allowed for a while). A language is called monotone when there is a monotone automaton that recognizes it.    
This notion of monotonicity is not immediately related to the current article, because our notion of monotonicity is not defined on automata (nor on neural networks) but on behaviors, which formalize semantics separate from the actual computation mechanism.
Moreover, the positive neural networks studied in this article may reuse the same global state while processing an input string, where a global state is defined as the set of currently activated neurons. For example, the empty global state could occur multiple times while processing an input string, even when this empty global state has precursor states and successor states that are not empty.

\section{Preliminaries}
\label{sec:prelim}

\subsection{Finite Automata and Regular Languages}
\label{sub:automata}

We recall the definitions of finite automata and regular languages~\cite{sipser_book2006}.
An \emph{alphabet} $\alp$ is a finite set.
A \emph{string $\str$ over $\alp$} is a finite sequence of elements from $\alp$. 
The empty string corresponds to the empty sequence.
We also refer to the elements of a string as its \emph{symbols}.
A \emph{language $\lang$ over $\alp$} is a set of strings over $\alp$. 
Languages can be finite or infinite.

The length of a string $\str$ is denoted $\len\str$.
For each $i\in\set{1,\ldots,\len\str}$, we write $\str_i$ to denote the symbol of $\str$ at position $i$.
We use the following string notation:
$
    \str = \mstr{\str_1, \ldots, \str_{\len\str}}.
$
For each $i\in\set{1,\ldots,\len\str}$, let $\prefix\str i$ denote the prefix $\mstr{\str_1,\ldots,\str_i}$.

A \emph{(finite) automaton} is a tuple $\nfa=\nfatup$ where 
\begin{itemize}
    \item $\states$ is a finite set of states;
    \item $\alp$ is an alphabet;
    \item $\trf$ is the transition function, mapping each pair $(q,\symm)\in\states\times\alp$ to a subset of $\states$;%
        \footnote{Importantly, this subset could be empty.}
    \item $\startstate$ is the start state, with $\startstate\in\states$; and,
    \item $\acceptstates\subseteq\states$ is the set of accepting states. 
\end{itemize}

Let $\str=\mstr{\str_1,\ldots,\str_n}$ be a string over $\alp$.
We call a sequence of states $q_1,\ldots,q_{n+1}$ of $\nfa$ a \emph{run of $\nfa$ on $\str$} if the following conditions are satisfied:
\begin{itemize}
    \item $q_1=\startstate$; and,
    \item $q_i\in\trf(q_{i-1}, \str_{i-1})$ for each $i\in\set{2,\ldots,n+1}$.
\end{itemize}
We say that the run $q_1,\ldots,q_{n+1}$ is \emph{accepting} if $q_{n+1}\in\acceptstates$.
We say that the automaton $\nfa$ \emph{accepts} $\str$ if there is an accepting run of $\nfa$ on $\str$.%
    \footnote{Our definition of automata omits the special symbol $\jump$, that can be used to visit multiple states in sequence without simultaneously reading symbols from the input string. This feature can indeed always be removed from an automaton, without increasing the number of states~\cite{hopcroft-ullman1979}.}
Automaton $\nfa$ could be \emph{nondeterministic}: for the same input string $\str$, there could be multiple accepting runs.
See also Remark~\ref{remark:parallel} below.

We define the language $\lang$ over $\alp$ that is \emph{recognized} by $\nfa$: language $\lang$ is the set of all strings over $\alp$ that are accepted by $\nfa$.
Now, a language is said to be \emph{regular} if it is recognized by an automaton.

\begin{remark}
    \label{remark:parallel}
    We call an automaton $\nfa=\nfatup$ \emph{deterministic} if $\ssize{\trf(q,\symm)}=1$ for each $(q,\symm)\in\states\times\alp$, i.e., the successor state is uniquely defined for each combination of a predecessor state and an input symbol.    
    Nondeterministic automata are typically smaller and easier to understand compared to deterministic automata~\cite{sipser_book2006}.
    Moreover, if $\nfa$ is nondeterministic then it represents \emph{parallel} computation. 
    To see this, we can define an alternative but equivalent semantics for $\nfa$ as follows~\cite{sipser_book2006}.
    The \emph{parallel run} of $\nfa$ on an input string $\str=\mstr{\str_1,\ldots,\str_n}$ over $\alp$ is the sequence
    \[
        \pset_1,\ldots,\pset_{n+1},
    \]
    where $\pset_1=\set{\startstate}$ and $\pset_i=\set{q_i\in\states\mid\exists q_{i-1}\in\pset_{i-1}\text{ with }q_i\in\trf(q_{i-1},\str_{i-1})}$ for each $i\in\set{2,\ldots,n+1}$.       
    We say that $\nfa$ \emph{accepts $\str$ under the parallel semantics} if the last state set of the parallel run contains an accepting state.
    It can be shown that the parallel semantics is equivalent to the semantics of acceptance given earlier.
    Because non-deterministic automata explore multiple states simultaneously at runtime, they appear to be a natural model for understanding parallel computation in neural networks (see Section~\ref{sub:lower}).    
    \qed
\end{remark}

\subsection{Behaviors}

We use behaviors to describe computations separate from neural networks.
Regarding notation, for a set $\jset$, let $\powset\jset$ denote the \emph{powerset} of $\jset$, i.e., the set of all subsets of $\jset$.

Let $\inr$ and $\onr$ be finite sets, whose elements we may think of as representing neurons. 
In particular, the elements of $\inr$ and $\onr$ are called \emph{input} and \emph{output} neurons respectively.
Now, a \emph{behavior $\be$ over input set $\inr$ and output set $\onr$} is a function that maps each nonempty string over alphabet $\powset\inr$ to a subset of $\onr$.
Regarding terminology, for a string $\str$ over $\powset\inr$ and an index $i\in\set{1,\ldots,\len\str}$, the symbol $\str_i$ says which input neurons are \emph{active} at (discrete) time $i$.
Note that multiple input neurons can be concurrently active.

For an input string $\str=\mstr{\str_1,\ldots,\str_n}$ over $\powset\inr$, the behavior $\be$ implicitly defines the following output string $\strB=\mstr{\strB_1,\ldots,\strB_{n+1}}$ over $\powset\onr$: 
\begin{itemize}
    \item $\strB_1=\emptyset$, and
    \item $\strB_i=\be(\prefix\str{i-1})$ for each $i\in\set{2,\ldots,n+1}$.
\end{itemize}
So, the behavior has access to the preceding input history when producing each output symbol. 
But an output symbol is never based on future input symbols.

\subsection{Monotone-regular Behaviors}

Let $\inr$ be a set of input neurons.
We call a language $\lang$ over alphabet $\powset\inr$  \emph{founded} when each string of $\lang$ is nonempty and has a nonempty subset of $\inr$ for its first symbol.
Also, for two strings $\str$ and $\strB$ over $\powset\inr$, we say that $\str$ \emph{embeds} $\strB$ if $\str$ has a suffix $\strC$ with $\len\strC=\len\strB$ such that $\strB_i\subseteq \strC_i$ for each $i\in\set{1,\ldots,\len\strB}$.
Note that $\strB$ occurs at the \emph{end} of $\str$.
Also note that a string embeds itself according to this definition.

Let $\be$ be a behavior over an input set $\inr$ and an output set $\onr$.
We call $\be$ \emph{monotone-regular} if for each output neuron $x\in\onr$ there is a founded regular language $\nrl x$ such that for each nonempty input string $\str$ over $\powset\inr$,
\[
    x\in\be(\str) \Leftrightarrow \text{$\str$ embeds a string $\strB\in\nrl x$}.
\]
Intuitively, the regular language $\nrl x$ describes the patterns that output neuron $x$ reacts to.
So, the meaning of neuron $x$ is the recognition of language $\nrl x$.
We use the term \emph{monotone} to indicate that $\nrl x$ is recognized within surrounding superfluous activations of input neurons, through the notion of embedding.
The restriction to founded regular languages expresses that outputs do not emerge spontaneously, i.e, the activations of output neurons are given the opportunity to witness at least one activation of an input neuron.

\begin{remark}
    Let $\nfa$ be an automaton that recognizes a founded regular language over $\powset\inr$.
    When reading the symbol $\emptyset$ from the start state of $\nfa$, we may only enter states from which it is impossible to reach an accepting state; otherwise the recognized language is not founded. See also Lemma~\ref{lem:founded} in Section~\ref{sub:lower}.
    \qed
\end{remark}

\begin{remark}
    The definition of monotone-regular behaviors fuses the separate notions of monotonicity and (founded) regular languages. 
    It also seems possible to define monotone-regular behaviors as those behaviors that are both monotone and regular. 
    However, in the formalization of regular behaviors, the regular language of each output neuron $x$ likely has to describe the entire input strings upon which $x$ is activated (at the end).
    This is in contrast to the current formalization of monotone-regular behaviors, where the (founded) regular language $\nrl x$ could be very small, describing only the patterns that $x$ is really trying to recognize, even when those patterns are embedded in larger inputs. 
    The current formalization is therefore more insightful for our construction in the expressivity lower bound (Theorem~\ref{theo:lower}), where we convert an automaton for $\nrl x$ to a neural network that serves as a pattern recognizer for output neuron $x$.
    The current formalization of monotone-regular behaviors allows the pattern recognizer to be as small as possible.
    \qed
\end{remark}

\subsection{Positive Neural Networks}
\label{sub:networks}

We define a neural network model that is related to previous discrete time models~\cite{sima_neuromata1998,sima_survey2003}, but with the following differences: 
    we have no inhibition, and we consider multiple input neurons that are allowed to be concurrently active.

Formally, a \emph{(positive) neural network} $\nw$ is a tuple $\nwtup$, where
\begin{itemize}
    \item $\inr$, $\onr$, and $\aux$ are finite and pairwise disjoint sets, containing respectively the \emph{input} neurons, the \emph{output} neurons, and the \emph{auxiliary} neurons;%
        \footnote{Auxiliary neurons are also sometimes called \emph{hidden} neurons~\cite{sima_survey2003}.}
    
    \item we let 
    \begin{align*}
        \conn\nw= & (\inr\times\onr)\cup(\inr\times\aux)\cup(\aux\times\onr)\\
            & \cup\set{(x,y)\in\aux\times\aux\mid x\neq y}
    \end{align*}
    be the set of possible connections; and,
    
    \item the function  $\weightinit$ is the \emph{weight function} that maps each $(x,y)\in\conn\nw$ to a value in $\interval 01$.
\end{itemize}
Note that there are direct connections from the input neurons to the output neurons.
The weight $0$ is used for representing missing connections.
Intuitively, the role of the auxiliary neurons is to provide working memory while processing input strings.    
For example, the activation of an auxiliary neuron could mean that a certain pattern was detected in the input string. 
Auxiliary neurons can recognize increasingly longer patterns by activating each other~\cite{elman_structure1990,kappel_markov2014}. 
We refer to Section~\ref{sec:results} for constructions involving auxiliary neurons.

We introduce some notations for convenience. 
If $\nw$ is understood from the context, for each $x\in\inr\cup\onr\cup\aux$, we abbreviate 
\[
    \pre x=\set{y\in\inr\cup\aux \mid (y,x)\in\conn\nw\text{ and }\weightinit(y,x)> 0}
\]
and 
\[
    \post x=\set{y\in\onr\cup\aux \mid (x,y)\in\conn\nw\text{ and }\weightinit(x,y)> 0}.
\]
We call $\pre x$ the set of \emph{presynaptic} neurons of $x$, and $\post x$ the set of \emph{postsynaptic} neurons of $x$.

\subsubsection{Operational Semantics}
\label{sub:opsem}

Let $\nw=\nwtup$ be a neural network.
We formalize how $\nw$ processes an input string $\str$ over $\powset\inr$.
We start with the intuition.

\paragraph*{Intuition}    
    We do $\len\str$ steps, called \emph{transitions}, to process all symbols of $\str$.
    At each time $i\in\set{1,\ldots,\len\str}$, also referred to as transition $i$, we show the input symbol $\str_i$ to $\nw$.
    Specifically, an input neuron $x\in\inr$ is active at time $i$ if $x\in\str_i$.
    Input symbols could activate auxiliary and output neurons. Auxiliary neurons could in turn also activate other auxiliary neurons and output neurons.
    Each time a neuron of $\inr\cup\aux$ becomes active, it conceptually emits a signal.    
    The signal emitted by a neuron $x$ at time $i$ travels to all postsynaptic neurons $y$ of $x$, and such received signals are processed by $y$ at the next time $i+1$.
    Each signal that is emitted by $x$ and received by a postsynaptic neuron $y$ has an associated weight, namely, the weight on the connection from $x$ to $y$.
    Subsequently, a postsynaptic neuron $y$ emits a (next) signal if the sum of all received signal weights is larger than or equal to a firing threshold. 
    The firing threshold in our model is $1$ for all neurons.
    All received signals are immediately discarded when proceeding to the next time.
    In the formalization below, the conceptual signals are not explicitly represented, and instead the transitions directly update sets of activated neurons.
    
    
\paragraph*{Transitions}
A \emph{transition of $\nw$} is a triple $(\activeX i,\symm,\activeX j)$ where
    $\activeX i\subseteq\onr\cup\aux$ and $\activeX j\subseteq\onr\cup\aux$ are two sets of \emph{activated} neurons, $\symm\in\powset\inr$ is an input symbol,    
    and where 
    \[
        \activeX j=%
        \set{%
        y\in\onr\cup\aux\mid%
            \sum_{z\in\pre y\cap(\activeX i\cup\symm)} \weightinit(z,y)\geq 1%
        }.
    \]    
We call $\activeX i$ the \emph{source set}, $\activeX j$ the \emph{target set}, and $\symm$ the symbol that is \emph{read}.%
    \footnote{We include output neurons in transitions only for technical convenience. It is indeed not essential to include output neurons in the source and target sets, because output neurons have no postsynaptic neurons and their activation can be uniquely deduced from the activations of auxiliary neurons and input neurons.}

\paragraph*{Run}
The \emph{run $\run$ of $\nw$ on input $\str$} is the unique sequence of $\len\str$ transitions for which
\begin{itemize}
    \item the transition with ordinal $i\in\set{1,\ldots,\len\str}$ reads input symbol $\str_i$;
    \item the source set of the first transition is $\emptyset$;
    \item the target set of each transition is the source set of the next transition.    
\end{itemize}
Note that $\run$ defines $\len\str+1$ sets of activated neurons, including the first source set.
We define the \emph{output of $\nw$ on $\str$}, denoted $\nw(\str)$, as the set $\activeX{}\cap\onr$ where $\activeX{}$ is the target set of the last transition in the run of $\nw$ on $\str$.

It is possible to consider the behavior $\be$ \emph{defined} by $\nw$: for each nonempty input string $\str$, we define $\be(\str)=\nw(\str)$.
So, like a behavior, a neural network implicitly transforms an input string $\str=\mstr{\str_1,\ldots,\str_n}$ over $\powset\inr$ to an output string $\strB=\mstr{\strB_1,\ldots,\strB_{n+1}}$ over $\powset\onr$:
\begin{itemize}
    \item $\strB_1=\emptyset$, and
    \item $\strB_i=\nw(\prefix\str{i-1})$ for each $i\in\set{2,\ldots,n+1}$.
\end{itemize}

\subsubsection{Design Choices}
    
    We discuss the design choices of the formalization of positive neural networks. 
    Although the model is simple, we have some preferences in how to formalize it.
    
    First, the reason for not having connections from output neurons to auxiliary neurons is for simplicity, and so that proofs can more cleanly separate the roles of neurons.
    However, connections from output neurons to auxiliary neurons can be simulated in the current model by duplicating each output neuron as an auxiliary neuron, including its presynaptic weights. 
    
    We exclude self-connections on neurons, i.e., connections from a neuron to itself, because such connections might be less common in biological neural networks. 
    
    The connection weights are restricted to the interval $\interval 01$ to express that there is a maximal strength by which any two neurons can be connected.    
    In biological neural networks, the weight contributed by a single connection, which abstracts a set of synapses, is usually much smaller than the firing threshold~\cite{gerstner_book2014}.
    For technical simplicity (cf.\ Section~\ref{sec:results}), however, the weights in our model are relatively large compared to the firing threshold.%
        \footnote{The largest weight is $1$, which is equal to the firing threshold; so, a neuron could in principle become activated when only one of its presynaptic neurons is active.}
    Intuitively, such larger weights represent a hidden assembly of multiple neurons that become active concurrently, causing the resulting sum of emitted weights to be large~\cite{maass_lowerbounds1996}.
    
    We use a normalized firing threshold of $1$ for simplicity. Another choice of positive firing threshold could in principle be compensated for by allowing connection weights larger than $1$.    

\subsection{Implementing Behaviors, with Delay}

Let $\nw=\nwtup$ be a neural network.
We say that a behavior $\be$ is \emph{compatible} with $\nw$ if $\be$ is over input set $\inr$ and output set $\onr$.

Delay is a standard notion in the expressivity study of neural networks~\cite{sima_neuromata1998,sima_survey2003}.
We say that $\nw$ \emph{implements a compatible behavior $\be$ with delay $k\in\nat$} when for each input string $\str$ over $\powset\inr$,
\begin{itemize}
    \item if $\len\str\leq k$ then $\nw(\str)=\emptyset$;%
        \footnote{If $k=0$ then this condition is immediately true because we consider no input strings with length zero.}
    and,
    \item if $\len\str > k$ then $\nw(\str) = \be(\prefix\str m)$ where $m=\len\str-k$.
\end{itemize}
Intuitively, delay is the amount of additional time steps that $\nw$ needs before it can conform to the behavior. This additional time is provided by reading more input symbols.%
    \footnote{Suppose $\nw$ implements $\be$ with delay $k$. Let $\str$ be an input string with $\len\str> k$. If we consider $\str$ as the entire input to the network $\nw$, then the last $k$ input symbols of $\str$ may be arbitrary; those symbols only provide additional time steps for $\nw$ to compute $\be(\prefix\str m)$ where $m=\len\str-k$.}
Note that a zero delay implementation corresponds to $\nw(\str)=\be(\str)$ for all input strings $\str$.

Letting $\be$ be the behavior defined by $\nw$, note that $\nw$ implements $\be$ with zero delay.

\begin{remark}
    \simaWcitet{} show that a neural network recognizing a regular language with delay $k$ over a single input neuron can be transformed into a (larger) neural network that recognizes the same language with delay $1$.
    An assumption in the construction, is that the delay $k$ in the original network is caused by paths of length $k$ from the input neuron to output neurons.    
    
    The definition of delay in this article is purely semantical: we only look at the timing of output neurons. There could be delay on output neurons, even though there might be direct connections from input neurons to output neurons, because output neurons might cooperate with auxiliary neurons (which might introduce delays).
    
    For completeness, we note that our construction in the expressivity lower bound (Theorem~\ref{theo:lower}) does not create direct connections from input neurons to output neurons, and thereby incurs a delay of a least one time unit; but we show that it is actually a delay of precisely one time unit. This construction therefore resembles the syntactical assumption by \citet{sima_neuromata1998}.
    \qed
\end{remark}


\section{Expressivity Results}
\label{sec:results}

Our goal is to better understand what positive neural networks can do.
Within the discrete-time framework of monotone-regular behaviors, we propose an upper bound on expressivity in Section~\ref{sub:upper}; a lower bound on expressivity in Section~\ref{sub:lower}; and, in Section~\ref{sub:separation}, examples showing that these bounds do not coincide.
This separation arises because our analysis takes into account the delay by which a neural network implements a monotone-regular behavior. 
It turns out that an implementation of zero delay exists for some monotone-regular behaviors, but not for other monotone-regular behaviors.
A delay of one time unit is sufficient for implementing all monotone-regular behaviors.
As an additional result, we present in Section~\ref{sub:also-zero} a large class of monotone-regular behaviors that can be implemented with zero delay.
If we would ignore delay, however, our upper and lower bound results (Sections~\ref{sub:upper} and \ref{sub:lower} respectively) intuitively say that the class of positive neural networks captures the class of monotone-regular behaviors: the behavior defined by a positive neural network is monotone-regular, and each monotone-regular behavior can be implemented by a positive neural network.

\subsection{Upper Bound}
\label{sub:upper}

Our expressivity upper bound says that only monotone-regular behaviors can be expressed by positive neural networks. This result is in line with the result by \simaWcitet{}, with the difference that we now work with multiple input neurons and the notion of monotonicity.

\begin{theorem}
    \label{theo:upper}
    The behaviors defined by positive neural networks are monotone-regular.
\end{theorem}
\begin{proof}    
    Intuitively, because a positive neural network only has a finite number of subsets of auxiliary neurons to form its memory, the network behaves like a finite automaton. Hence, as is well-known, the performed computation can be described by a regular language~\cite{sima_neuromata1998}. 
    An interesting novel aspect, however, is monotonicity, meaning that output neurons recognize patterns even when those patterns are embedded into larger inputs.
    
    Let $\nw=\nwtup$ be a positive neural network.    
    Let $\be$ denote the behavior defined by $\nw$.
    We show that $\be$ is monotone-regular.
    Fix some $x\in\onr$.
    We define a founded regular language $\nrl x$ such that for each input string $\str$ over $\powset\inr$ we have
    \[
        x\in\be(\str) \Leftrightarrow \text{$\str$ embeds a string $\strB\in\nrl x$}.
    \]  
    We first define a deterministic automaton $\nfa$.
    Let $\startstate$ and $\haltstate$ be two state symbols where 
            $\startstate\neq\haltstate$ and
            $\set{\startstate,\haltstate}\cap\powset{\onr\cup\aux}=\emptyset$.
    We call $\haltstate$ the \emph{halt state} because no useful processing will be performed anymore when $\nfa$ gets into state $\haltstate$ (see below).        
    We concretely define $\nfa=\nfatup$, where
    \begin{itemize}        
        \item $\states=\set{\startstate,\haltstate}\cup\powset{\onr\cup\aux}$;
        
        \item $\alp=\powset\inr$;        
       
        \item regarding $\trf$, for each $(q,\symm)\in\states\times\alp$,
        \begin{itemize}
            \item if $q=\startstate$ and $\symm=\emptyset$ then 
                $\trf(q,\symm)=\set{\haltstate}$;
                        
            \item if $q=\startstate$ and $\symm\neq\emptyset$ then
                $\trf(q,\symm)=\set{q'}$ where
                \[
                     q' = \set{y\in\onr\cup\aux\mid %
                                \sum_{z\in\pre y\cap \symm} \weightinit(z,y) \geq 1%
                            };
                \]
                                             
            \item if $q=\haltstate$ then
                $\trf(q,\symm)=\set{\haltstate}$;
               
            \item if $q\in\powset{\onr\cup\aux}$ then $\trf(q,\symm) = \set{q'}$ where
            \[
                q' = \set{y\in\onr\cup\aux\mid %
                            \sum_{z\in\pre y\cap(q\cup \symm)} \weightinit(z,y) \geq 1%
                        };
            \]            
        \end{itemize}            
             
        \item $\acceptstates = \set{q\in\powset{\onr\cup\aux}\mid x\in q}$.
    \end{itemize}
    The addition of state $\haltstate$ is to obtain a founded regular language: strings accepted by $\nfa$ start with a nonempty input symbol.       
    We define $\nrl x$ as the founded regular language recognized by $\nfa$.%
         \footnote{The construction in this proof does not necessarily result in the smallest founded regular language $\nrl x$. The activation of $x$ is based on seeing patterns embedded in a suffix of the input, but our construction also includes strings in $\nrl x$ that are extensions of such patterns with arbitrary prefixes (starting with a nonempty input symbol).}    
    
    Next, let $\str=\mstr{\sym 1,\ldots,\sym n}$ be a string over $\powset\inr$.
    We show that
    \[
        x\in\be(\str) \Leftrightarrow \text{$\str$ embeds a string $\strB\in\nrl x$.}        
    \]
    
    \paragraph*{Direction 1}
    Suppose $x\in\be(\str)$.
    Because no neurons are activated on empty symbols, we can consider the smallest index $k\in\set{1,\ldots,n}$ with $\sym k\neq\emptyset$.
    Let $\strB=\mstr{\sym k,\ldots,\sym n}$. 
    Clearly $\str$ embeds $\strB$.
    Note that $\nw(\strB)=\nw(\str)$, implying $x\in\nw(\strB)$.
    When giving $\strB$ as input to automaton $\nfa$, we do not enter state $\haltstate$ since $\strB$ starts with a nonempty input symbol.
    Subsequently, $\nfa$ faithfully simulates the activated neurons of $\nw$. The last state $q$ of $\nfa$ reached in this way, corresponds to the last set of activated neurons of $\nw$ on $\strB$. Since $x\in\nw(\strB)$, we have $q\in\acceptstates$, causing $\strB\in\nrl x$, as desired.

    \paragraph*{Direction 2}
    Suppose $\str$ embeds a string $\strB\in\nrl x$.
    Because $\strB\in\nrl x$, there is an accepting run of $\nfa$ on $\strB$, where the last state is an element $q\in\powset{\onr\cup\aux}$ with $x\in q$. 
    Since $\nfa$ faithfully simulates $\nw$, we have $x\in\nw(\strB)$.
    Because the connection weights of $\nw$ are nonnegative, if we would extend $\strB$ with more activations of input neurons both before and during $\strB$, like $\str$ does, then at least the neurons would be activated that were activated on just $\strB$. Hence, $x\in\nw(\str)$, as desired.
    
    \paragraph*{Remark}
    We did not define $\states=\onr\cup\aux$ because, when reading an input symbol, the activation of a neuron depends in general on multiple presynaptic auxiliary neurons. That context information might be lost when directly casting neurons as automaton states, because an automaton state is already reached by combining just one predecessor state with a new input symbol.
\end{proof}

The following example demonstrates that an implementation with zero delay is at least achievable for some simple monotone-regular behaviors.
In Section~\ref{sub:also-zero} we will also see more advanced monotone-regular behaviors that can be implemented with zero delay.
\begin{example}
    \label{ex:chain}
    Let $\be$ be a monotone-regular behavior over an input set $\inr$ and an output set $\onr$ with the following assumption: for each $x\in\onr$, the founded regular language $\nrl x$ contains just one string.
    The intuition for $\be$, is that a simple chain of auxiliary neurons suffices to recognize increasingly larger prefixes of the single string, and the output neuron listens to the last auxiliary neuron and the last input symbol. There is no delay.
    
    We now define a positive neural network $\nw=\nwtup$ to implement $\be$ with zero delay. For simplicity we assume $\ssize\onr=1$, and we denote $\onr=\set{x}$; we can repeat the construction below in case of multiple output neurons, and the partial results thus obtained can be placed into one network.    
    Denote $\nrl x=\set{\mstr{\sym 1,\ldots,\sym n}}$, where $\sym 1\neq\emptyset$.
    If $n=1$ then we define $\aux=\emptyset$ and, letting $m=\ssize{\sym 1}$, we define $\weightinit(u,x)=\myfrac 1m$ for each $u\in\sym 1$; all other weights are set to zero.
    We can observe that $\nw(\str)=\be(\str)$ for each input string $\str$ over $\powset\inr$.
    
    Now assume $n\geq 2$. We define $\aux$ to consist of the pairwise different neurons $y_1,\ldots,y_{n-1}$, with the assumption $x\notin\aux$.
    Intuitively, neuron $y_1$ should detect symbol $\sym 1$. Next, for each $i\in\set{2,\ldots,n-1}$, neuron $y_i$ is responsible for detecting symbol $\sym i$ when the prefix $\mstr{\sym 1,\ldots, \sym{i-1}}$ is already recognized; this is accomplished by letting $y_i$ also listen to $y_{i-1}$. 
    We specify weight function $\weightinit$ as follows, where any unspecified weights are assumed to be zero:
    \begin{itemize}
        \item For neuron $y_1$, letting $m=\ssize{\sym 1}$, we define $\weightinit(u,y_1)=\myfrac 1m$ for each $u\in\sym 1$;
        
        \item For neuron $y_i$ with $i\in\set{2,\ldots,n-1}$, letting $m=\ssize{\sym i}+1$, we define $\weightinit(u,y_i)=\myfrac 1m$ for each $u\in\set{y_{i-1}}\cup\sym i$;
        
        \item For neuron $x$, letting $m=\ssize{\sym n}+1$, we define $\weightinit(u,x)=\myfrac 1m$ for each $u\in\set{y_{n-1}}\cup\sym n$.
    \end{itemize}
    Also for the case $n\geq 2$, we can observe that $\nw(\str)=\be(\str)$ for each input string $\str$ over $\powset\inr$.
    \qed
\end{example}

\subsection{Lower Bound}
\label{sub:lower}

The expressivity lower bound (Theorem~\ref{theo:lower}~below) complements the expressivity upper bound (Theorem~\ref{theo:upper}).
We first introduce some additional terminology and definitions.
\subsubsection{\Suitable\ Automata}

The construction in the expressivity lower bound is based on translating automata to neural networks. 
The Lemmas below allow us to make certain technical assumptions on these automata, making the translation to neural networks more natural.

We say that an automaton $\nfa=\nfatup$ contains a \emph{self-loop} if there is a pair $(q,\symm)\in\states\times\alp$ such that $q\in\trf(q,\symm)$.
The following Lemma tells us that self-loops can be removed:
\begin{lemma}
    \label{lem:selfloops}
    Every regular language recognized by an automaton $\nfaX 1$ is also recognized by an automaton $\nfaX 2$ that \romI\ contains no self-loops, and \romII\ uses at most double the number of states of $\nfaX 1$.%
        \footnote{Intuitively, the quantification of the number of states indicates that in general $\nfaX 2$ preserves the nondeterminism of $\nfaX 1$.}
\end{lemma}
\begin{proof}
    Denote $\nfaX 1=\nfatupX 1$.
    The idea is to duplicate each state involved in a self-loop, so that looping over the same symbol is still possible but now uses two states.
    Let $\loops$ be the set of all states of $\nfaX 1$ involved in a self-loop:
    \[
    \loops = \set{q\in\statesX 1\mid \exists\symm\in\alpX 1 \text{ with }q\in\trfX 1(q,\symm)}.
    \]
    Let $f$ be an injective function that maps each $q\in\loops$ to a new state $f(q)$ outside $\statesX 1$.
    To construct $\nfaX 2$, we use the state set $\statesX 1 \cup \set{f(q)\mid q\in\loops}$; the same start state as $\nfaX 1$; and, the accepting state set $\acceptstatesX 1 \cup \set{f(q)\mid q\in\acceptstatesX 1\cap\loops}$. 
    For the new transition function, each pair $(q,\symm)\in\statesX 1\times\alpX 1$ with $q\in\trfX 1(q,\symm)$ is mapped to $\set{f(q)}\cup(\trfX 1(q,\symm)\setminus\set{q})$, and $(f(q),\symm')$ is mapped to $\trfX 1(q,\symm')$ for each $\symm'\in\alpX 1$.%
        \footnote{If $q\in\trfX 1(q,\symm')$ then we can go back from the new state $f(q)$ to the old state $q$ by reading symbol $\symm'$.}
    An odd number of repetitions over symbol $\symm$ is possible because we have copied all outgoing transitions of $q$ to $f(q)$.    
    All other pairs $(q,\symm)\in\statesX 1\times\alpX 1$ with $q\notin\trfX 1(q,\symm)$ are mapped as before.
\end{proof}

For founded regular languages, Lemma~\ref{lem:founded} (below), tells us that the symbol $\emptyset$ does not have to be read from the start state.
Intuitively, this last assumption means that activated states of an automaton can be simulated by neurons: the activations of input neurons in the first input symbol can be propagated through the neural network to keep track of any further progress, even if subsequent input symbols are empty.
\begin{lemma}
    \label{lem:founded}
    Letting $\inr$ be an input set, every founded regular language over $\powset\inr$ recognized by an automaton $\nfaX 1$ is also recognized by an automaton $\nfaX 2=\nfatupX 2$ where \romI\ $\trfX 2(\startstateX 2,\emptyset)=\emptyset$, and \romII\ $\nfaX 2$ has the same states as $\nfaX 1$.       
\end{lemma}
\begin{proof}
    The automaton $\nfaX 2$ is almost exactly the same as $\nfaX 1$, except that the state-symbol combination $(\startstateX 2,\emptyset)$ is mapped by the transition function to $\emptyset$, i.e., it is impossible to read the empty symbol from the start state. 
    We can immediately see that all accepting runs of $\nfaX 2$ are also accepting runs of $\nfaX 1$ because $\nfaX 1$ includes all transition possibilities of $\nfaX 2$.
    
    For the other direction, towards a contradiction, suppose there is an accepting run $q_1,\ldots,q_{n+1}$ of $\nfaX 1$ on a string $\str=\mstr{\sym 1,\ldots,\sym n}$ but this run is not an accepting run of $\nfaX 2$.
    Because in $\nfaX 2$ we have only removed the option to read symbol $\emptyset$ from the start state, there has to be some $i\in\set{1,\ldots,n}$ with $\sym i=\emptyset$ and $q_i$ is the start state (of $\nfaX 1$, and $\nfaX 2$).
    Now, note that the state sequence $q_i,\ldots,q_{n+1}$ is an accepting run of $\nfaX 1$ on the suffix $\strB=\mstr{\sym i,\ldots,\sym n}$.
    But since $\sym i=\emptyset$, automaton $\nfaX 1$ would not recognize a founded regular language, which is a contradiction.
\end{proof}

Let $\nfa$ be as above. A state $q\in\states$ is said to be \emph{\useful} if there is string $\str$ over $\alp$ and a run of $\nfa$ on $\str$ in which $q$ appears; this run does not have to be accepting. 
Clearly, every regular language recognized by an automaton $\nfaX 1$ is also recognized by an automaton $\nfaX 2$ that keeps only the \useful\ states of $\nfaX 1$.

Letting $\inr$ be an input set, and letting $\nfa$ be an automaton that recognizes a founded regular language over $\powset\inr$, we call $\nfa$ \emph{\suitable} if 
\begin{itemize}
    \item $\nfa$ contains no self-loops;    
    \item $\nfa$ does not read symbol $\emptyset$ from its start state; and,
    \item $\nfa$ contains only \useful\ states.
\end{itemize}
By applying Lemmas~\ref{lem:selfloops} and \ref{lem:founded} in order, any automaton recognizing a founded regular language can be converted to a \suitable\ one that recognizes the same language; and, the number of states is at most doubled compared to the original automaton (through Lemma~\ref{lem:selfloops}).

For a \suitable\ automaton $\nfa=\nfatup$, we define the \emph{pair set} of $\nfa$, denoted $\pairset\nfa$, as the following set
\[
    \set{(q,\symm)\in\states\times\alp\mid
    q\neq\startstate\text{ and }
    \exists q'\in\states\text{ with }q\in\trf(q',\symm)}.
\] 
In words: the pair set contains the combinations in $\nfa$ of a non-start state and an incoming symbol to that state.

Now, let $\be$ be a monotone-regular behavior over an input set $\inr$ and an output set $\onr$.
An \emph{\implementation} for $\be$ is a function $\nfaref$ mapping each $x\in\onr$ to a \suitable\ automaton $\nfaref(x)$ that recognizes a founded regular language $\nrl x$ over $\powset\inr$ such that for each input string $\str$ over $\powset\inr$,
\[
    \monoregexpr x.
\] 
Intuitively, an \implementation\ for $\be$ is a prototype implementation that can later be converted to a neural network.
The \emph{total pair count} of $\nfaref$, denoted $\paircount\nfaref$, is defined as
\[
    \paircount\nfaref = \sum_{x\in\onr}\ssize{\pairset{\nfaref(x)}}.
\]

\subsubsection{Lower Bound Result}

\begin{theorem}
    \label{theo:lower}
    Every monotone-regular behavior $\be$ can be implemented by a positive neural network with delay $1$.
    In particular, each \implementation\ $\nfaref$ for $\be$ can be converted to a positive neural network that implements $\be$ with delay $1$ and that has $\paircount\nfaref$ auxiliary neurons.%
        \footnote{Intuitively, the number of auxiliary neurons indicates that in general the constructed neural network preserves the nondeterminism, and thus the parallelism, of the automata in $\nfaref$.}
\end{theorem}
\begin{proof}            
    Let $\be$ be a monotone-regular behavior over an input set $\inr$ and an output set $\onr$.
    Let $\nfaref$ be an \implementation\ for $\be$.
    For each output neuron $x$, we translate automaton $\nfaref(x)$ to a neural network. 
    Roughly speaking, we translate state-symbol pairs of the automaton to neurons.
    A novel aspect, is that each input symbol in our model consists of multiple input neurons. For this reason, our simulation of an automaton state by a neuron uses a nontrivial definition of presynaptic weights allowing us to simultaneously express \romI\ an ``or'' over auxiliary neurons that provide working memory, and \romII\ an ``and'' over all input neurons mentioned in an input symbol.
    We use only rational weights.
    There is a delay of one time unit in the construction because the output neuron $x$ listens to neurons that simulate accept states of $\nfaref(x)$.%
        \footnote{An automaton itself does not introduce delay on string acceptance. In the construction of a neural network, however, all the different accept states should essentially be tunneled through a single output neuron. This requires in general a delay of one time unit (cf.\ Section~\ref{sub:separation}).}
    See also the later Remark~\ref{remark:preprocessor}.
    The construction below is illustrated in Example~\ref{ex:transform}.
    
    For simplicity, we assume $\ssize\onr=1$, and we denote $\onr=\set{x}$; for the case of multiple output neurons, the construction given below can be repeated, and the neural networks thus obtained can be united to form the overall desired network.   
    Let $\nfa=\nfaref(x)$ and denote $\nfa=\nfatup$ where $\alp=\powset\inr$.
    Recall that $\nfa$ is \suitable.
    
    \paragraph*{Positive neural network}
    We now incrementally define the desired positive neural network $\nw=\nwtup$ to implement $\be$ with delay $1$.    
    
    \subparagraph*{Auxiliary neurons}
    First, we define the set of auxiliary neurons:
    \[
        \aux = \pairset\nfa,       
    \]
    where $\pairset\nfa$ is the pair set of $\nfa$ as defined above.
    Intuitively, an auxiliary neuron $(q,\symm)$, where always $q\neq\startstate$, represents the automaton state $q$ reached by reading input symbol $\symm$ from some previous state.    
    We define the set $\triggers\subseteq\aux$ of \emph{trigger neurons}:
    \[
        \triggers = \set{(q,\symm)\in\aux \mid q\in\trf(\startstate,\symm)}.
    \]
    Intuitively, the neurons in $\triggers$ are the first auxiliary neurons that become activated by the input; these neurons simulate the event of reading an input symbol from the start state of automaton $\nfa$.
    Note that for each $(q,\symm)\in\triggers$ we have $\symm\neq\emptyset$ because $\trf(\startstate,\emptyset)=\emptyset$ by assumption on $\nfa$.
    
    For each $(q,\symm)\in\aux\setminus\triggers$, we define the set $\context(q,\symm)$ of \emph{context} neurons of $(q,\symm)$ as follows:
    \[
        \context(q,\symm) = \set{(q',\symm')\in\aux\mid q\in\trf(q',\symm)}.
    \]
    Intuitively, $\context(q,\symm)$ is the set of auxiliary neurons that recognize prefixes of the strings that neuron $(q,\symm)$ should recognize, i.e., $\context(q,\symm)$ is the working memory from the viewpoint of $(q,\symm)$.        
    In the definition of $\context(q,\symm)$, there is no relationship between the symbols $\symm$ and $\symm'$.       
    Note that for each $(q,\symm)\in\aux\setminus\triggers$, the set $\context(q,\symm)$ is always nonempty because $\nfa$ contains only \useful\ states.%
        \footnote{Indeed, since $(q,\symm)\in\aux$, there is a \useful\ state $q'\in\states$ with $q\in\trf(q',\symm)$. But $(q,\symm)\notin\triggers$ implies $q'\neq\startstate$, causing $(q',\symm')\in\aux$ for some $\symm'\in\powset\inr$. Hence, $(q',\symm')\in\context(q,\symm)$.}
    
    \subparagraph*{Weights}
    The design of the connection weights is an intricate part of the construction. For this reason, we spend sufficient attention to the underlying design process.
    Suppose we have an auxiliary neuron $y=(q,\symm)$ that should listen to a context $\context(q,\symm)=\set{z_1,\ldots,z_m}$ of auxiliary neurons and to the input symbol $\symm=\set{u_1,\ldots,u_n}$.
    We desire weights $\worname$ and $\wandname$, where $\worname$ is assigned to each connection $(z_i,y)$ with $i\in\set{1,\ldots,m}$ and $\wandname$ to each connection $(u_j,y)$ with $j\in\set{1,\ldots,n}$, such that the following three properties are satisfied: 
        \romI\ $y$ is not activated if all of $\set{z_1,\ldots,z_m}$ are activated but not yet all of $\set{u_1,\ldots,u_n}$; 
        \romII\ $y$ is already activated if at least one $z\in\set{z_1,\ldots,z_m}$ is activated while all of $\set{u_1,\ldots,u_n}$ are activated; and,
        \romIII\ $y$ is not activated if only all of $\set{u_1,\ldots,u_n}$ are activated.
    This assignment of weights corresponds to the earlier announced ``or'' and ``and'', over $\set{z_1,\ldots,z_m}$  and $\set{u_1,\ldots,u_n}$ respectively.
    
    The above desired properties $\romI$, $\romII$, and $\romIII$ are satisfied by the following weight functions $\worname$ and $\wandname$ that are parameterized by the set cardinalities $m$ and $n$, denoting $\natzero=\nat\setminus\set{0}$,
    \begin{align*}                
        & \worname:\natzero\times\natzero\to\interval 01:\quad \wor mn = \myfrac{1}{(n\mult m + 1)},\\
        & \wandname:\natzero\times\natzero\to\interval 01:\quad \wand mn = \myfrac{m}{(n\mult m + 1)}.
    \end{align*}
    The design of these functions is documented in Appendix~\ref{app:deriv-weights}. The satisfaction of the desired properties is now formalized by the following observations:
    
    \begin{samepage}
    \begin{claim}
        \label{claim:weights}
        Letting $m,n\in\natzero$,
        \begin{itemize}                        
            \item $m\mult\wor mn + (n-1)\mult\wand mn < 1$;
            \item $\wor mn + n\mult\wand mn \geq 1$;
            \item $n\mult\wand mn < 1$.
        \end{itemize}
    \end{claim} 
    \end{samepage}
    
    Next, we can define the weights for all connections. We define the weight function $\weightinit$ from the perspective of the neurons in $\aux\cup\set{x}$, where any unmentioned weights are assumed to be zero:
    \begin{itemize}
        \item for the output neuron $x$, and each $(q,\symm)\in\aux$ where $q$ is an accepting state of $\nfa$ (i.e., $q\in\acceptstates$), we define
        \[
            \weightinit((q,\symm),x) = 1;
        \]
    
        \item for each $(q,\symm)\in\triggers$ and each $y\in\symm$, letting $n=\ssize\symm$, we define
        \[
            \weightinit(y,(q,\symm)) = \myfrac 1n;
        \]
        
        \item for each $(q,\symm)\in\aux\setminus\triggers$ with $\symm=\emptyset$, for each $y\in\context(q,\symm)$, we define
        \[
            \weightinit(y, (q,\symm)) = 1;
        \]
        
        \item for each $(q,\symm)\in\aux\setminus\triggers$ with $\symm\neq\emptyset$, letting $m=\ssize{\context(q,\symm)}$ and $n=\ssize\symm$, for each $y\in\context(q,\symm)$, we define
        \[
            \weightinit(y,(q,\symm)) = \wor mn,
        \]
        and for each $z\in\symm$, we define
        \[
            \weightinit(z,(q,\symm)) = \wand mn;
        \]
        note in this case that $m>0$ and $n>0$.
    \end{itemize}  
    Intuitively, the role of neurons $(q,\symm)\in\aux\setminus\triggers$ with $\symm=\emptyset$ is to propagate past memories forward in time, without requiring new activations of any input neurons.
    
    \paragraph*{Correctness}    
    We show that $\nw$ implements $\be$ with a delay of one time unit. 
    Let $\str=\mstr{\str_1,\ldots,\str_n}$ be an input string over $\powset\inr$.
    If $n=1$ then $\nw(\str)=\emptyset$, as desired, because the output neuron $x$ only listens to auxiliary neurons (that represent accepting states), which makes it impossible for $x$ to become activated on a string with just one symbol.
    Henceforth, suppose $n\geq 2$. We show that $\nw(\str)=\be(\prefix\str{n-1})$.
    
    \subparagraph*{Direction 1}
    Suppose $x\in\nw(\str)$.
    The activation of $x$ means that there is a maximal chain of auxiliary neurons 
    \[
        \qsX 1,\ldots,\qsX k,
    \]
    that becomes activated when showing $\str$ to $\nw$ (with $k\geq 1$), where $\qsX 1$ is a trigger neuron; $\qsX 2$, \ldots, $\qsX k$ are non-trigger auxiliary neurons; $\qsX{i-1}$ is a presynaptic neuron of $\qsX i$ for each $i\in\set{2,\ldots,k}$; and, $\qsX k$ has activated $x$ while the last input symbol $\str_n$ was shown.
    Let $\strB=\mstr{\sym 1,\ldots,\sym k}$.
    By design of the presynaptic weights of the auxiliary neurons (cf.\ Claim~\ref{claim:weights}), we know that the symbols $\sym 1$, \ldots, $\sym k$ effectively occur in $\str$, and more particularly that $\prefix\str{n-1}$ embeds $\strB$.    
    Next, we show that $\nfa$ accepts $\strB$. Then, since $\be$ is monotone-regular, the embedding of $\strB$ into $\prefix\str{n-1}$ implies $x\in\be(\prefix\str{n-1})$.
    
    Based on the above sequence of auxiliary neurons, the state sequence $\startstate,q_1,\ldots,q_k$ forms an accepting run of $\nfa$ on $\strB$: 
    \begin{itemize}
        \item $q_1\in\trf(\startstate,\sym 1)$ because $\qsX 1$ is a trigger neuron; 
        
        \item for each $i\in\set{2,\ldots,k}$, we have $q_{i}\in\trf(q_{i-1},\sym{i})$ because $\qsX{i-1}$ is a presynaptic neuron of $\qsX{i}$;%
            \footnote{From the definition of presynaptic neuron, we know that the connection from $\qsX{i-1}$ to $\qsX i$ has a strictly positive weight. This weight could only have been defined if $\qsX{i-1}\in\context{\qsX{i}}$.}
        
        \item $q_k$ must be an accepting state, because we assumed that neuron $\qsX k$ has activated $x$.
    \end{itemize}
    
    \subparagraph*{Direction 2}    
    Suppose $x\in\be(\prefix\str{n-1})$.
    Because $\be$ is monotone-regular, $\prefix\str{n-1}$ embeds a string $\strB$ that is accepted by $\nfa$.
    Denote $\strB=\mstr{\sym 1,\ldots,\sym k}$.
    We consider an accepting run $\startstate,q_1,\ldots,q_k$ of $\nfa$ on $\strB$.
    The string $\strB$ can be chosen so that $\startstate\notin\set{q_1,\ldots,q_k}$.%
        \footnote{If $\startstate=q_i$ for some $i\in\set{1,\ldots,k}$ then $q_i,q_{i+1},\ldots,q_k$ is an accepting run on the suffix $\strB'=\mstr{\sym{i+1},\ldots,\sym{k}}$, and we could instead focus on the smaller string $\strB'$ that is also embedded into $\prefix\str{n-1}$.}
    We now consider the following sequence of auxiliary neurons: $\qsX 1,\ldots,\qsX k$.%
        \footnote{These are valid auxiliary neurons because
            \romI\ $\startstate\notin\set{q_1,\ldots,q_k}$ by assumption; and,
            
            \romII\ because $\startstate,q_1,\ldots,q_k$ is an accepting run, we have $q_1\in\trf(\startstate,\sym 1)$ and $q_i\in\trf(q_{i-1},\sym i)$ for each $i\in\set{2,\ldots,k}$.}
    We show that this sequence of auxiliary neurons becomes active in the last $k$ steps of $\nw$ on input $\prefix\str{n-1}$.
    Let $\activeX 1,\ldots,\activeX n$ be the sequence of sets of activated neurons while running $\nw$ on input $\prefix\str{n-1}$, where $\activeX 1=\emptyset$. 
    We show (by induction) for each $i\in\set{1,\ldots,k}$ that $\qsX i\in\activeX{n-k+i}$.
    This results in $\qsX k\in\activeX{n}$, and because $\qsX k$ simulates an accepting state, on the full string $\str$ we thus obtain $x\in\nw(\str)$, as desired.
    
    Before we continue, note that the embedding of $\strB$ into $\prefix\str{n-1}$ concretely means $\sym i\subseteq\str_{n-1-k+i}$ for each $i\in\set{1,\ldots,k}$.
    For the base case, we see that $\qsX 1$ is a trigger neuron because $q_1\in\trf(\startstate,\sym 1)$. So, $\sym 1\subseteq\str_{n-k}$ implies $\qsX 1\in\activeX{n-k+1}$.
    
    For the inductive step, we assume $\qsX{i-1}\in\activeX{n-k+i-1}$ where $i\in\set{2,\ldots,k}$. We show that $\qsX i\in\activeX{n-k+i}$.
    If $\qsX i$ is a trigger neuron then a similar reasoning applies as in the base case, using that $\sym i\subseteq\str_{n-1-k+i}$.
    If $\qsX i$ is not a trigger neuron then $\qsX{i-1}\in\context(q_i,\sym i)$ because $q_i\in\trf(q_{i-1},\sym i)$ and $q_{i-1}\neq\startstate$, and we distinguish between the following two cases:
    \begin{itemize}
        \item Suppose $\sym i=\emptyset$. Then the connection weight from $\qsX{i-1}$ to $\qsX i$ was set to $1$, and the activation $\qsX{i-1}\in\activeX{n-k+i-1}$ implies the activation $\qsX i\in\activeX{n-k+i}$.
        
        \item Suppose $\sym i\neq\emptyset$. In that case, the presynaptic weight design of $\qsX i$ with functions $\worname$ and $\wandname$ (cf.\ Claim~\ref{claim:weights}), applied to the presynaptic activations $\qsX{i-1}\in\activeX{n-k+i-1}$ and $\sym i\subseteq\str_{n-1-k+i}=\str_{n-k+i-1}$, gives the activation $\qsX i\in\activeX{n-k+i}$.
    \end{itemize}
\end{proof}

\begin{example}
    \label{ex:transform}
    We illustrate the construction of the proof of Theorem~\ref{theo:lower}.
    Let $\inr$ consist of four distinct input neurons $a$, $b$, $c$, and $d$. Let $\onr=\set{x}$.
    We define the following input symbols: $\sym 1=\set{a,b,c}$, $\sym 2=\set{b,c}$, and $\sym 3=\set{a,d}$.
    
    Consider the \suitable\ automaton $\nfa$ depicted in Figure~\ref{fig:nfa}, that recognizes a founded regular language over $\powset\inr$; we denote this language as $\nrl x$.%
        \footnote{In Figure~\ref{fig:nfa}, we use the standard notations~\cite{hopcroft-ullman1979,sipser_book2006}: the start state has an entering arrow with no source, and accepting states are indicated with double circles.}
    Language $\nrl x$ is infinite because of the loop between states $q_1$ and $q_2$ over symbol $\sym 2$.
    In particular, $\nrl x$ contains all strings of the form $\mstr{\sym 1,\sym 2^*,\sym 3}$, where $\sym 2^*$ denotes an arbitrary number of repetitions of symbol $\sym 2$.
    Let $\be$ be the monotone-regular behavior over $\inr$ and $\onr$ defined by $\nrl x$: for each input string $\str$ over $\powset\inr$,
    \[
        \monoregexpr x.
    \]
    
    Applying the transformation in the proof of Theorem~\ref{theo:lower} to automaton $\nfa$ results in the positive neural network $\nw$ depicted in Figure~\ref{fig:nw}, where input neurons are indicated by boxes and the nonzero (rational) edge weights are written at the end of a connection. 
    Auxiliary neuron $(q_1,\sym 1)$ is the only trigger neuron; it listens for symbol $\sym 1$.
    Note that the loop between states $q_1$ and $q_2$ of $\nfa$ is preserved as a loop between the auxiliary neurons $(q_1,\sym 2)$ and $(q_2,\sym 2)$.
    We can also see, for example, that the neuron $(q_3,\sym 3)$ is only activated at time $t\in\nat$ when at time $t-1$ both input neurons $a$ and $d$ are active and at least one of the auxiliary neurons $(q_1,\sym 1)$, $(q_2,\sym 2)$, and $(q_1,\sym 2)$; these auxiliary neurons may be viewed as working memory, representing the recognition of prefixes of the desired strings.
    \qed
\end{example}

\begin{figure}
    \begin{center}
    \begin{subfigure}[b]{0.45\textwidth}  
        \centering
        \includegraphics[height=0.4\textheight]{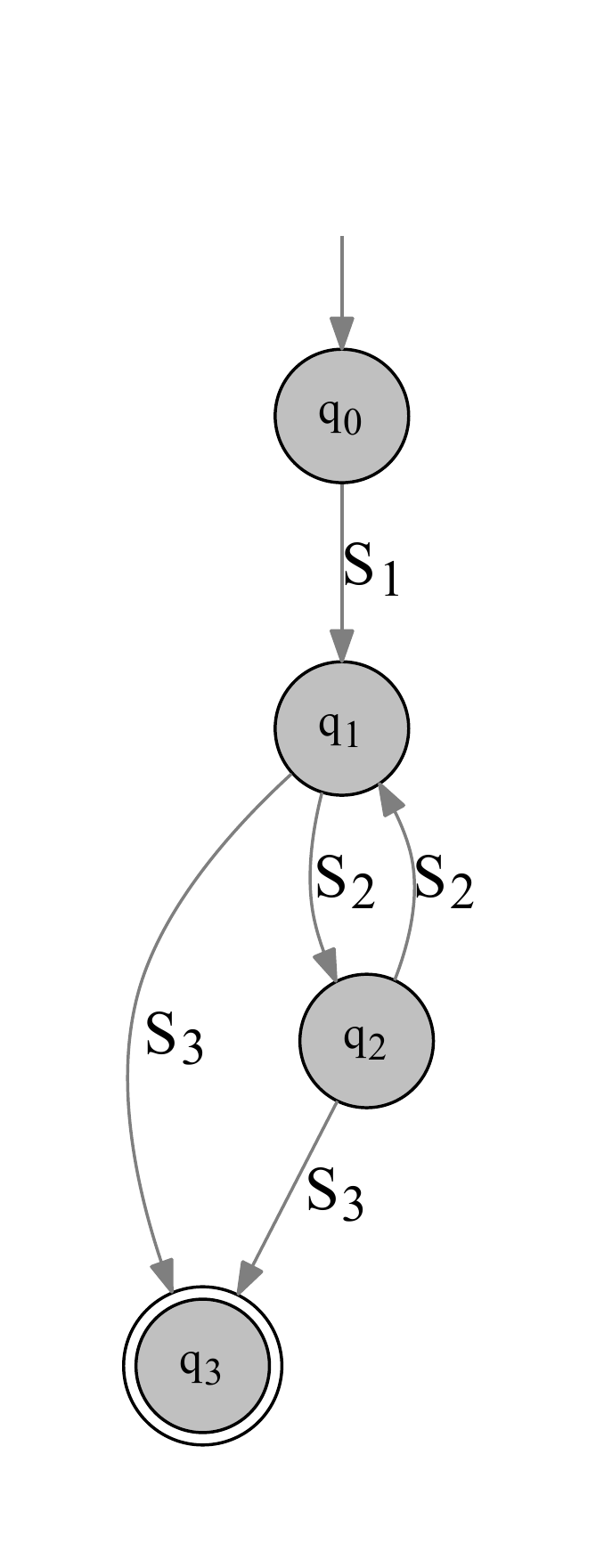}
        \caption{A \suitable\ automaton recognizing a founded regular language. The input neurons are $a$, $b$, $c$, and $d$; the considered input symbols are $\sym 1=\set{a,b,c}$, $\sym 2=\set{b,c}$, and $\sym 3=\set{a,d}$.}
        \label{fig:nfa}
    \end{subfigure}
    ~
    \begin{subfigure}[b]{0.45\textwidth}
        \hspace{-1.5cm}
        \includegraphics[height=0.4\textheight]{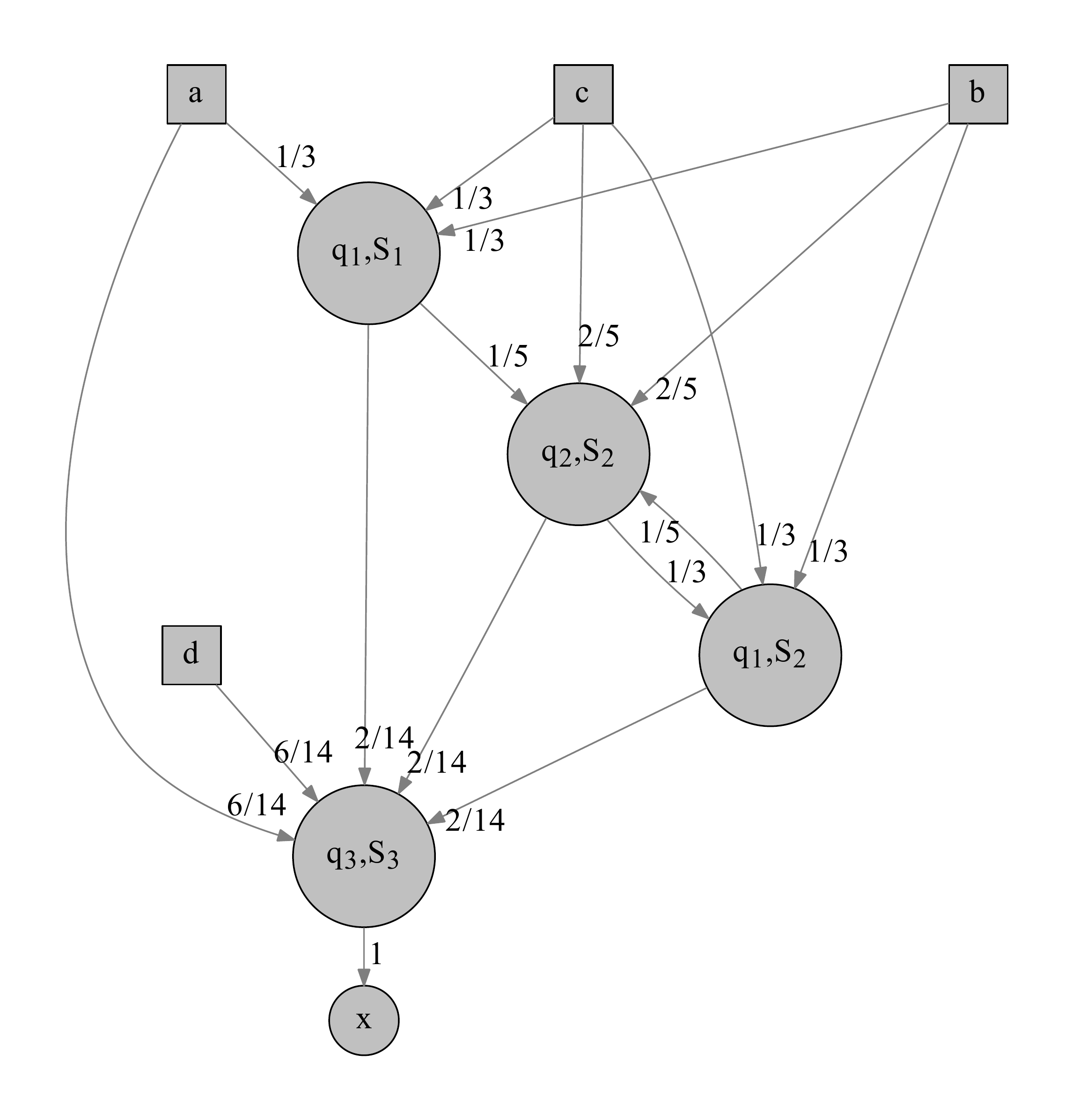}
        \caption{The positive neural network obtained from the automaton in Figure~\ref{fig:nfa}. The boxes represent the input neurons $a$, $b$, $c$, and $d$. The output neuron is $x$. The remaining neurons are auxiliary.}
        \label{fig:nw}
    \end{subfigure}
    \end{center} 
    \caption{The automaton and positive neural network of Example~\ref{ex:transform}.}\label{fig:example}
\end{figure}

\begin{remark}
    \label{remark:preprocessor}
    In the proof of Theorem~\ref{theo:lower}, it is possible to replace the or-and construction of weight functions $\worname$ and $\wandname$ by a two-stage process, at the cost of an additional delay of one time unit. If we ignore this additional delay, the resulting construction is similar to the one described by \simaWcitet{} in their Theorem~4.1, for the setting with one input neuron, with the difference that we only use positive weights and are thus expressing monotone-regular behaviors.        
    Concretely, for each symbol $\symm\in\powset\inr$ used by the automaton, with $\symm\neq\emptyset$, we introduce a \emph{preprocessor} neuron $y_\symm$ having the following presynaptic weight for each $u\in\symm$, where $n=\ssize\symm$:
    \[
        \weightinit(u,y_\symm) = \myfrac 1n.
    \]    
    So, neuron $y_\symm$ will only be activated when all neurons of $\symm$ are activated.
    Next, each auxiliary neuron $(q,\symm)\in\aux$ with $\symm\neq\emptyset$ is configured to read the preprocessor neuron $y_\symm$ instead of the input neurons in $\symm$ directly:
    \begin{itemize}
        \item if $(q,\symm)\in\triggers$ then we define $\weightinit(y_\symm,(q,\symm)) = 1$;
        \item if $(q,\symm)\in\aux\setminus\triggers$ with $\symm=\emptyset$ then for each $z\in\context(q,\symm)$ we define $\weightinit(z,(q,\symm))=1$ as before;
        \item if $(q,\symm)\in\aux\setminus\triggers$ with $\symm\neq\emptyset$, letting $m=\ssize{\context(q,\symm)}$, we define 
        \[
            \weightinit(y_\symm,(q,\symm))=\myfrac m{(m+1)},
        \]
        and for each $z\in\context(q,\symm)$,
        \[
            \weightinit(z,(q,\symm)) = \myfrac 1{(m+1)}.
        \]
    \end{itemize}
    The total implementation delay now becomes two time units: \romI\ trigger neurons listen to the above preprocessor neurons, and \romII\ the output neurons listen to auxiliary neurons that simulate accept states as before.
    We should point out, however, that the construction by \simaWcitet{} only incurs a delay of one time unit because in their setting there is only one input neuron; so, in that setting, all the above preprocessor neurons can be conceptually merged into the single input neuron.
    \qed
\end{remark}

\subsection{Separation}
\label{sub:separation}

Regarding the expressivity of positive neural networks, the upper bound (Theorem~\ref{theo:upper}) and the lower bound (Theorem~\ref{theo:lower}) do not coincide.
Indeed, as illustrated by the following two examples, there are simple monotone-regular behaviors that can not be implemented with zero delay.
The main intuition in these examples, is that the fast reaction speed demanded by zero delay forces too much responsibility on the output neuron, causing this neuron to be erroneously activated. Each example illustrates a different kind of error.

\begin{example}
    \label{ex:wrong-symbol}
    Let $\sym 1$ and $\sym 2$ be two disjoint sets of neurons with $\ssize{\sym 1}\geq 2$ and $\ssize{\sym 2}\geq 2$.
    Let $\inr=\sym 1\cup\sym 2$ and $\onr=\set{x}$.
    Let $\nrl x$ be the following founded regular language over $\powset\inr$:
    \[
        \nrl x = \set{\mstr{\sym 1}, \mstr{\sym 2}}.
    \]
    So, $\nrl x$ is a finite language containing two one-symbol strings.%
        \footnote{An automaton recognizing $\nrl x$ could have two accepting states $q_1$ and $q_2$ besides the start state $\startstate$: reading symbol $\sym i$ from $\startstate$ leads to $q_i$ for $i\in\set{1,2}$.}
    Let $\be$ be the following monotone-regular behavior over $\inr$ and $\onr$ defined by $\nrl x$: for each input string $\str$ over $\powset\inr$, we define
    \[
        \be(\str) = 
            \begin{cases}
                \set{x}& \text{if $\str$ embeds a string $\strB\in\nrl x$;}\\
                \emptyset& \text{otherwise.}
            \end{cases}
    \]
    
    We show that there is no positive neural network that implements $\be$ with zero delay.
    Towards a contradiction, suppose there is such a neural network $\nw$.
    We show that the connections from $\sym 1$ to $x$ and the connections from $\sym 2$ to $x$ interfere with each other, causing $x$ to also be triggered on wrong input symbols.
    
    Because $\nw$ implements $\be$ with zero delay, we have $\nw(\str)=\be(\str)$ for all input strings $\str$ over $\powset\inr$.    
    In particular, $\nw(\mstr{\sym 1})=\set{x}$ and $\nw(\mstr{\sym 2})=\set{x}$.
    These fast output reactions imply that neuron $x$ does not rely on auxiliary neurons, and instead reads input neurons directly.
    So,
    \begin{align}
        & \sum_{u\in \sym 1}\weightinit(u,x)\geq 1 \text{, and} \label{eq:threshold1} \\
        & \sum_{u\in \sym 2}\weightinit(u,x)\geq 1 \label{eq:threshold2}.
    \end{align}    
    We distinguish between the following cases:
    \begin{itemize}
        \item Suppose there exist some $y\in\sym 1$ and $z\in\sym 2$ such that 
        \[
            \weightinit(y,x) + \weightinit(z,x) \geq 1.
        \]
        Define the symbol $\symm=\set{y,z}$. Note that $\symm\in\powset\inr$.
        Because $\ssize{\sym 1}\geq 2$ and $\ssize{\sym 2}\geq 2$, we have $\sym 1\not\subseteq\symm$ and $\sym 2\not\subseteq\symm$.
        Please note that by choice of $y$ and $z$,
        \[
            \sum_{u\in\symm}\weightinit(u,x)\geq 1.
        \]
        So, $\nw(\mstr{\symm})=\set{x}$. 
        But the string $\mstr{\symm}$ does not embed a string from $\nrl x$, giving $\be(\mstr{\symm})=\emptyset$.
        Hence, $\nw(\mstr{\symm})\neq\be(\mstr{\symm})$, which is a contradiction.
        
        \item If the first case does not hold, then we can choose some $y\in\sym 1$ and $z\in\sym 2$ for which 
        \[
            \weightinit(y,x) + \weightinit(z,x) < 1.
        \]
        Define the symbol $\symm=\inr\setminus\set{y,z}$. 
        Note that $\symm\in\powset\inr$.
        Because $y\in\sym 1$ and $z\in\sym 2$, we have $\sym 1\not\subseteq\symm$ and $\sym 2\not\subseteq\symm$.
        Moreover,
        \[
            \sum_{u\in\symm}\weightinit(u,x) = %
                \sum_{u\in\sym 1}\weightinit(u,x) %
                + \sum_{u\in\sym 2}\weightinit(u,x) %
                - \weightinit(y,x) - \weightinit(z,x).
        \]
        By using inequalities \eqref{eq:threshold1} and \eqref{eq:threshold2} from above, and $\weightinit(y,x)+\weightinit(z,x)<1$, we can further obtain:
        \begin{align*}
            \sum_{u\in\symm}\weightinit(u,x) &\geq 2 - (\weightinit(y,x)+\weightinit(z,x))\\
               & > 1.
        \end{align*}
        So, $\nw(\mstr{\symm})=\set{x}$.
        But the string $\mstr{\symm}$ does not embed a string from $\nrl x$, giving $\be(\mstr{\symm})=\emptyset$.
        Again, $\nw(\mstr{\symm})\neq\be(\mstr{\symm})$, which is a contradiction.
    \end{itemize}
    \qed
\end{example}

\begin{example}
    \newcommand{\w}[1]{w_{#1}} 
    \newcommand{\myaux}[1]{\aux_{#1}} 
    Let $\sym 1$, $\sym 2$, $\sym 3$, and $\sym 4$ be nonempty sets of neurons that are pairwise disjoint.
    Let $\inr=\bigcup_{i=1}^{4}\sym i$ and $\onr=\set{x}$.
    Let $\nrl x$ be the following founded regular language over $\powset\inr$:%
        \footnote{An automaton recognizing this language could splits its computation into two branches from the start state: one branch recognizes the string $\mstr{\sym1,\sym2}$ and the other branch recognizes the string $\mstr{\sym3,\sym4}$.}
    \[
        \nrl x = \set{\mstr{\sym 1, \sym 2}, \mstr{\sym 3, \sym 4}}.
    \]
    Let $\be$ be the monotone-regular behavior over $\inr$ and $\onr$ defined by $\nrl x$: for each input string $\str$ over $\powset\inr$,
    \[    
        \be(\str) = 
            \begin{cases}
                \set{x}& \text{if $\str$ embeds a string $\strB\in\nrl x$;}\\
                \emptyset& \text{otherwise.}
            \end{cases}
    \]
    
    We show there is no positive neural network that implements $\be$ with zero delay.
    Towards a contradiction, suppose there is such a network $\nw=\nwtup$.
    We show that $\nw$ erroneously activates the output neuron on the input string $\mstr{\sym 1,\sym 4}$ or on the input string $\mstr{\sym 3, \sym 2}$.
    Intuitively, the output neuron $x$ confuses the memory contexts emerging from symbols $\sym 1$ and $\sym 3$.
    
    Because $\nw$ implements $\be$ with zero delay, we have $\nw(\str)=\be(\str)$ for all input strings $\str$ over $\powset\inr$.
    In particular, $\nw(\mstr{\sym 1,\sym 2})=\set{x}$ and $\nw(\mstr{\sym 3,\sym 4})=\set{x}$.
    Let $\myaux 1\subseteq\aux$ denote the set of auxiliary neurons activated after reading the string $\mstr{\sym 1}$.        
    Similarly, let $\myaux 3\subseteq\aux$ denote the set of auxiliary neurons activated after reading the string $\mstr{\sym 3}$.
    Denote, for $i\in\set{1,3}$,
    \[
        \w i = \sum_{y\in\myaux i}\weightinit(y,x).
    \]
    Also denote, for $i\in\set{2,4}$,
    \[
        \w i = \sum_{y\in\sym i}\weightinit(y,x).
    \]    
    Now, the output activations $\nw(\mstr{\sym 1, \sym 2})=\set{x}$ and $\nw(\mstr{\sym 3,\sym 4})=\set{x}$ imply
    \begin{align*}
        & \w 1 + \w 2 \geq 1\text{, and}\\
        & \w 3 + \w 4 \geq 1.
    \end{align*}    
    We distinguish between the following cases:%
        \footnote{Although $\nw(\mstr{\sym 1,\sym 2})=\set{x}$ and $\nw(\mstr{\sym 2})=\be(\mstr{\sym 2})=\emptyset$ imply that $\w 1> 0$, the proof does not really use this fact. Similarly, $\w 3> 0$, but the proof does not use this fact.}
    \begin{itemize}
        \item Suppose $\w 1 + \w 4 \geq 1$.
            This implies $\nw(\mstr{\sym 1,\sym 4})=\set{x}$.
            But then $\nw(\mstr{\sym 1,\sym 4})\neq\be(\mstr{\sym 1,\sym 4})$, which is a contradiction.
            
        \item In the other case, we have $\w 1 + \w 4 < 1$.
        Together with $\w 3 + \w 4 \geq 1$ from above, we see that $\w 3 > \w 1$.
        Combining $\w 3>\w 1$ and $\w 1 + \w 2\geq 1$ from above, we obtain $\w 3 + \w 2\geq 1$.
        This implies $\nw(\mstr{\sym 3,\sym 2})=\set{x}$.
        But then $\nw(\mstr{\sym 3,\sym 2})\neq\be(\mstr{\sym 3,\sym 2})$, which is a contradiction.
    \end{itemize}
    \qed
\end{example}

\subsection{On Zero Delay}
\label{sub:also-zero}

The earlier Example~\ref{ex:chain} has provided a zero delay implementation for monotone-regular behaviors whose underlying founded regular language contains only one string.
Here, we present a larger class of monotone-regular behaviors that can be implemented with zero delay.
First, we call a regular language $\lang$ \emph{converging} if all strings in $\lang$ end with the same symbol.
The following result demonstrates that even monotone-regular behaviors whose underlying founded regular languages are infinite can sometimes be implemented with zero delay:
\begin{theorem}
    \label{theo:also}
    Every monotone-regular behavior where the founded regular language of each output neuron is also converging, can be implemented by a positive neural network with zero delay.
\end{theorem}
\begin{proof}      
    Let $\be$ be a monotone-regular behavior over an input set $\inr$ and an output set $\onr$ where the founded regular language of each output neuron is also converging.
    Let $\nfaref$ be an \implementation\ for $\be$.
    As in the proof of Theorem~\ref{theo:lower}, we fix some $x\in\onr$.
    Let $\nrl x$ be the language recognized by $\nfaref(x)$.    
    Denote $\nfaref(x)=\nfatup$, where $\alp=\powset\inr$.
    We can modify the construction in the proof of Theorem~\ref{theo:lower} as follows.
    
    First, we define the set $\conv$ of all state-symbol combinations that lead to an accepting state:
    \[
        \conv = \set{(q,\symm)\in\states\times\powset\inr\mid \trf(q,\symm)\cap\acceptstates\neq\emptyset}.
    \]    
    Because $\nrl x$ is converging, there is one symbol $\symm\in\powset\inr$ such that $\symm=\sym i$ for each $\qsX i\in\conv$.%
        \footnote{For each $\qsX i\in\conv$, there is an input string $\str$ over $\powset\inr$ and a run of $\nfaref(x)$ on $\str$ ending with $q_i$ because $q_i$ is a \useful\ state by assumption on $\nfaref(x)$. Since $\qsX i\in\conv$, the extension of $\str$ with $\sym i$ belongs to $\nrl x$. So, for any $\qsX 1\in\conv$ and $\qsX 2\in\conv$, there are strings in $\nrl x$ ending with $\sym 1$ and $\sym 2$; but convergence of $\nrl x$ implies $\sym 1=\sym 2$.}
    We refer to $\symm$ as the \emph{terminal} symbol.
    The only difference compared to the proof of Theorem~\ref{theo:lower}, is that we now let output neuron $x$ listen to \romI\ the symbol $\symm$ directly and \romII\ a different set $\outcontext$ of auxiliary neurons.
    Letting $\aux$ be the set of auxiliary neurons as defined in the proof of Theorem~\ref{theo:lower}, we define
    \[
        \outcontext = \set{(q,\symm')\in\aux\mid (q,\symm)\in\conv}.
    \]
    We now specify the presynaptic weights for $x$, depending on symbol $\symm$:
    \begin{itemize}
        \item Suppose $\symm=\emptyset$. 
        We still have $\outcontext\neq\emptyset$: there is always a string $\str\in\nrl x$ ending with $\symm$, for which there is an accepting run $q_1,\ldots,q_n,q_{n+1}$ where $q_{n+1}\in\trf(q_n,\symm)\cap\acceptstates$; and, $q_n\neq\startstate$ because $\nfaref(x)$ does not read $\symm=\emptyset$ from its start state, implying $(q_n,\symm')\in\outcontext$ for some $\symm'\in\powset\inr$.
        Now, for each $y\in\outcontext$, we define
        \[
            \weightinit(y,x)= 1.
        \]
        
        \item Suppose $\symm\neq\emptyset$. If $\outcontext=\emptyset$ then $x$ only has to detect symbol $\symm$; accordingly, letting $n=\ssize\symm$, for each $z\in\symm$, we define
        \[
            \weightinit(z,x) = \myfrac 1n.
        \]
        If $\outcontext\neq\emptyset$, then we reuse the or-and construction with weight functions $\worname$ and $\wandname$; concretely, letting $m=\ssize\outcontext$ and $n=\ssize\symm$, for $y\in\outcontext$, we define
        \[
            \weightinit(y,x)=\wor mn,            
        \]
        and for each $z\in\symm$, we define
        \[
            \weightinit(z,x)=\wand mn.
        \]
    \end{itemize}    
    All other connections from auxiliary neurons to $x$ are set to zero.
    So, instead of listening to auxiliary neurons that simulate accept states, the output neuron $x$ \romI\ listens to auxiliary neurons that simulate the states preceding accept states, and \romII\ also verifies that the terminal symbol $\symm$ effectively occurs.
\end{proof}

The following example demonstrates that the converse of Theorem~\ref{theo:also} does not hold, so we do not yet have a precise characterization of the monotone-regular behaviors that can be implemented with zero delay.
\begin{example}
    \label{ex:bias}
    Let $\sym 1=\set{a,b}$ and $\sym 2=\set{b,c}$ where $a$, $b$, and $c$ are pairwise different neurons.
    Let $\inr=\sym 1\cup\sym 2$ and $\onr=\set{x}$.
    Let $\nrl x$ be the following founded regular language over $\powset\inr$:
    \[
        \nrl x=\set{\mstr{\sym 1}, \mstr{\sym 2}}.
    \]
    Note that $\nrl x$ is not converging.
    Let $\be$ be the monotone-regular behavior over $\inr$ and $\onr$ defined by $\nrl x$: for each input string $\str$ over $\powset\inr$,
    \[
        \be(\str) = \begin{cases}
            \set{x} & \text{if $\str$ embeds a string $\strB\in\nrl x$;}\\
            \emptyset & \text{otherwise.}
        \end{cases}
    \]
    
    The following positive neural network $\nw=\nwtup$ implements $\be$ with zero delay: $\aux=\emptyset$, and
    \begin{align*}
        & \weightinit(a,x)=\myfrac 13,\\
        & \weightinit(b,x)=\myfrac 23,\\
        & \weightinit(c,x)=\myfrac 13.
    \end{align*}
    In contrast to Example~\ref{ex:wrong-symbol}, we can not fool this network to trigger $x$ on a wrong input symbol like $\set{a,c}$.
    That is because $\weightinit$ assigns a heavier weight to connection $(b,x)$, which renders the input neuron $b$ crucial for the activation of $x$.
    \qed
\end{example}


\section{Conclusion and Future Work}
\label{sec:conclusion}

We have studied the expressivity of positive neural networks with multiple input neurons.
Within the framework of monotone-regular behaviors, we have suggested both an upper and lower bound on the expressivity.
These bounds do not coincide when we take into account the delay by which a behavior is implemented.
We now discuss several avenues for further work.

\paragraph*{Single input neurons}

    If there is only a single input neuron, \simaWcitet{} show that all regular languages can be recognized by a neural network with a delay of one time unit.
    Our article has shown a similar result for monotone-regular behaviors, but in the case of multiple input neurons. 
    It might be interesting to better understand the relationship between these results. 
    
    Symbols over multiple input neurons could be translated to a single input neuron as follows: supposing there are $n$ ordered input neurons, each subset of input neurons can be represented as a binary code over $n$ bits. This way, each sequence of input symbols can be translated to a sequence of binary codes, and the resulting sequence may be viewed as a single bit string.
    However, this construction would increase output delay. Moreover, it is not clear if this technical construction can be achieved inside a positive neural network itself, because on every time step an entirely new symbol arrives over the multiple input neurons; the positive neural network might not be able to buffer the new symbols while it is translating the previous symbols.

\paragraph*{Characterizing zero delay}

    We have seen that seemingly simple monotone-regular behaviors already require a delay of one time unit (Section~\ref{sub:separation}).
    We have also made some first steps towards identifying the class of monotone-regular behaviors that can be implemented with zero delay (Section~\ref{sub:also-zero}).
    However, a precise characterization is missing.    
    Example~\ref{ex:bias} suggests that in case of multiple terminal symbols in the underlying regular languages, we could seek for an assignment of nonuniform weights to the input neurons. 
    Perhaps the existence of such nonuniform weights can be related to the syntactical properties of the accompanying automata.

\paragraph*{Minimal network size}

    Like previous complexity-theoretic analyses of neural networks~\cite{sima_survey2003}, one could examine what minimal number of auxiliary neurons is necessary for implementing certain monotone-regular behaviors.
    Note that a lower bound on the number of states in an automaton implementation of a behavior does not directly provide a lower bound on the number of neurons, because clever design of the weights could perhaps pack more functionality into fewer neurons than the number of automaton states (or symbol-state combinations). Such efficient implementations were previously studied, e.g.\ by \citet{horne_fsm1996} for the simulation of deterministic automata by recurrent neural networks.
    For positive neural networks, it might be possible to explore the relationship with (monotone) AND-OR boolean circuits, where \citet{alon_monotone1987} have previously obtained lower bounds on the number of gates (neurons) for implementing certain boolean functions.
        
    We should note, however, that some of the existing constructions, e.g.~\cite{horne_fsm1996} introduce delays in which the overall neural network would process incoming input symbols.
    To compare such constructions with the results regarding delay in this article, perhaps some of the constructed sub-circuits could be viewed as being computed instantaneously, and would thus not contribute to the overall delay.
    
\paragraph*{Inhibition}

    Previous works on the expressive power of neural networks have often assumed negative connection weights between neurons, allowing neurons to inhibit the activation of their postsynaptic neurons~\cite{sima_survey2003}.
    It is interesting to extend our work with this feature, but in such a way that it is still biologically plausible.
    In particular, one should make a distinction between excitatory and inhibitory neurons~\cite{gerstner_book2014}: the postsynaptic weights of excitatory neurons are always positive and the postsynaptic weights of inhibitory neurons are always negative.
    Both neuron types are used in winner-take-all circuits~\cite{kappel_markov2014}.
    
    As suggested by the findings of \simaWcitet{}, inhibitory neurons could allow the neural network to test for the explicit absence of input activations, lifting the expressive power to ``regular'' behaviors that, in contrast to monotone-regular behaviors, depend on very precise input symbols that are not embedded in surrounding input noise.
    For example, a neural network might activate an output neuron whenever the input symbol $\set{a,b,c}$ occurs in its pure form, i.e., no other input neurons are active besides $a$, $b$, and $c$.
    
    Another view, is that inhibitory neurons have a stabilizing effect, at least in a winner-take-all setting~\cite{kappel_markov2014}: inhibitory neurons let the most strongly recognized patterns survive; otherwise perhaps too many insignificant pattern pieces will be floating around in the limited working memory.    
    
    Possibly, multiple biologically plausible topologies with inhibition are possible. The expressivity of the resulting neural network models, including any results regarding delays, could strongly depend on the manner by which inhibitory and excitatory neurons are connected.   
    
\paragraph*{Noise and continuous time}

    Noise is an important aspect of real biological neurons~\cite{gerstner_book2014}, and it might be an important resource for expressing nondeterministic computations~\cite{maass_noise2014}.
    It would be interesting to see how the results regarding regular languages can be extended to this framework. One possibility is to study the quality by which a noisy positive neural network approximates a true monotone-regular behavior. Here, quality might be formalized as the probability of producing correct output activations given a certain probability distribution on the noise.
    
    Moreover, the model studied in this article is based on discrete time steps. 
    Again, real-world neurons do not obey this restriction, so it appears interesting to investigate if our results can be extended to a setting with continuous time.
    However, the restriction to discrete time steps may enable an understanding of neurons that operate in continuous time by focusing on the causal relationships between neuron activations.
    From this viewpoint, regular languages could also provide insights into the workings of neurons operating in continuous time.
    
\paragraph*{Learning}

    An important aspect of biological neurons is that they modify their presynaptic weights over time through a learning mechanism called STDP, that depends on the relative timing of neuron activations~\cite{gerstner_book2014}.%
        \footnote{The acronym ``STDP'' stands for spike-timing-dependent plasticity.}
    One could for example consider reward-modulated STDP, where connection weights are updated at some time point when the overall performance of the neural network has recently improved~\cite{gerstner_book2014}. 
    In a biologically plausible setting, it seems intriguing to understand how overall behavior and consciousness could emerge from dopamine neurons signaling reward to an organism~\cite{schultz_reward2013}.
    
\paragraph{Forbidding recurrent connections}

    Weak recurrent connections in biological neural networks might already be sufficient to provide an interaction of working memory with new inputs~\cite{buonomano2009}.
    So, pure looping behavior as needed in the recognition of regular languages might not be really needed by an organism.
    So, in a further expressivity study, one could simplify positive neural networks by forbidding recurrent connections.
    This way, only finite regular languages can be recognized.
    It seems interesting to understand the resulting model from a practical perspective.
    In particular, one might verify if the resulting networks are still useful for real-world tasks. 
    It seems that memories of larger stimuli require more neurons, and longer activation chains between those neurons.

\paragraph*{Sharing auxiliary neurons}

    The construction for the expressivity lower bound (Theorem~\ref{theo:lower}) builds a separate network of auxiliary neurons for each output neuron.
    In biological networks, multiple output neurons share a pool of auxiliary neurons~\cite{buonomano2009}.
    It seems interesting to understand the impact of sharing on the behaviors exhibited by the individual output neurons.

\paragraph*{Multiple interconnected networks}
    In this article, we have investigated the expressiveness of single networks where all neurons are directly connected to each other.
    However, when the number of neurons increases, the number of direct connections increases quadratically. This would become impractical to implement in biological neural networks. 
    Indeed, one hypothesis is that the brain is composed of many small networks that are connected strongly internally, but perhaps only weakly externally~\cite{kappel_markov2014}.
    It is interesting to understand how such an organization of the connections influences the expressivity.

\subsection*{Acknowledgments}
The first author thanks Robert~Brijder for suggestions regarding the formalization of finite automata.

\bibliographystyle{apalike}

\appendix

\addtocontents{toc}{\protect\setcounter{tocdepth}{1}}

\section{Design of the Weights (Claim~\ref{claim:weights})}
\label{app:deriv-weights}

Denote $\natzero=\nat\setminus\set{0}$.
Let $m\in\natzero$ and $n\in\natzero$.
Suppose we have two sets $Y$ and $Z$ with $m=\ssize Y$ and $n=\ssize Z$. Both sets should form the presynaptic neurons of a neuron $x$. 
We want to find weights $w_1$ and $w_2$, to be assigned to the neurons in $Y$ and $Z$ respectively, such that 
\begin{enumerate}[1)]
    \item \label{enu:below} $m\mult w_1 + (n-1)\mult w_2 < 1$;
    \item \label{enu:above} $w_1 + n\mult w_2 \geq 1$;
    \item \label{enu:below2} $n\mult w_2 < 1$.
\end{enumerate}
Condition~\ref{enu:below} expresses that all neurons from $Z$ should be activated before $x$ may be activated, regardless of how many neurons in $Y$ are activated.
Condition~\ref{enu:above} expresses that if all neurons in $Z$ are activated then a single neuron from $Y$ suffices to activate $x$; but Condition~\ref{enu:below2} stipulates that at least one neuron of $Y$ should be activated.
So, neuron $x$ requires all neurons of $Z$ and just a single neuron from $Y$.
Our design of such weights is based on a denominator $f\in\natzero$:
\begin{align*}
    & w_1 = \myfrac 1f\text{,}\\
    & w_2 = \myfrac{(1-\myfrac 1f)}{n}.
\end{align*}
We see that Condition~\ref{enu:above} is satisfied for any $f\in\natzero$:
\begin{align*}
    \myfrac 1f + n\left(\myfrac{(1-\myfrac 1f)}n\right) 
        &= \myfrac 1f + (1- \myfrac 1f)\\
        &= 1 \geq 1.        
\end{align*}
Also, Condition~\ref{enu:below2} is satisfied for any $f\in\natzero$:
\begin{align*}
    n\left(\myfrac{(1-\myfrac 1f)}{n}\right) = 1-\myfrac 1f < 1.
\end{align*}
For Condition~\ref{enu:below}, we solve for $f$:
\begin{align*}
    m\mult w_1 + (n-1)\mult w_2 &< 1 ;\\
    \myfrac mf + (1-\myfrac 1f)\braces{\frac{n-1}n} &< 1 ;\\
    %
    %
    %
    %
    %
    %
    %
    %
    &\vdots \\
    f &> n(m-1) + 1.
\end{align*}
So, we can choose $f=n\mult m + 1$.%
    \footnote{Because $n>0$, we can make the following derivation: $m-1< m$; $n(m-1)< n\mult m$; $n(m-1) + 1 < n\mult m +1$.}


\end{document}